%% file: main.tex
\documentclass{article}

    \usepackage[nonatbib,final]{neurips_2021}

\usepackage[utf8]{inputenc} 
\usepackage[T1]{fontenc}    
\usepackage{hyperref}       
\usepackage{url}            
\usepackage{booktabs}       
\usepackage{amsfonts}       
\usepackage{nicefrac}       
\usepackage{microtype}      
\usepackage{xcolor}         

\include{header_neurips}

\title{Universal Off-Policy Evaluation}

\author{%
    Yash Chandak \\
    University of Massachusetts\\
    \And
    Scott Niekum \\
    University of Texas Austin \\
    \And
    Bruno Castro da Silva \\
    University of Massachusetts\\
    \And
    Erik Learned-Miller \\
    University of Massachusetts\\
    \And
    Emma Brunskill \\
    Stanford University \\
    \And
    Philip S. Thomas \\
    University of Massachusetts
}

\begin{document}

\maketitle

\begin{abstract}
When faced with sequential decision-making problems, it is often useful to be able to predict what would happen if decisions were made using a new policy. Those predictions must often be based on data collected under some previously used decision-making rule. 
Many previous methods enable such \emph{off-policy} (or counterfactual) estimation of the \textit{expected} value of a performance measure called the \emph{return}. 
In this paper, we take the first steps towards a \emph{\underline{un}iversal \underline{o}ff-policy estimator} (UnO)---one that provides off-policy estimates and high-confidence bounds for \emph{any} parameter of the return distribution.
We use UnO for estimating and simultaneously bounding the mean, variance, quantiles/median, inter-quantile range, CVaR, and the entire cumulative distribution of returns.
Finally, we also discuss UnO's applicability in various settings, including fully observable, partially observable (i.e., with unobserved confounders), Markovian, non-Markovian, stationary, smoothly non-stationary, and  discrete distribution shifts. 

\end{abstract}

\section{Introduction}

Problems requiring sequential decision-making are ubiquitous \citep{barto2017some}. 
When online experimentation is costly or dangerous, it is essential to conduct off-policy evaluation before deploying a new policy; that is, one must leverage existing data collected using some policy $\beta$ (called a behavior policy) to evaluate a performance metric of another policy $\pi$ (called the evaluation policy).
For problems with high stakes, such as in terms of health \citep{liao2020off} or financial assets \citep{theocharous2015ad},
 it is also crucial to provide high-confidence bounds on the desired performance metric to ensure reliability and safety. 

Perhaps the most widely studied performance metric in the off-policy setting is the expected return \citep{SuttonBarto2}.
However, this metric can be limiting for many problems of interest.
Safety-critical applications, such as automated healthcare, require minimizing the chances of risk-prone outcomes, and so performance metrics such as value at risk (VaR) or conditional value at risk (CVaR) are more appropriate \citep{keramati2020being,brown2020bayesian}. 
By contrast, applications like online recommendations are subject to noisy data and call for robust metrics like the median and other quantiles \citep{altschuler2019best}.
In order to improve user experiences, applications involving direct human-machine interaction, such as robotics and autonomous driving, focus on minimizing uncertainty in their outcomes and thus use metrics like variance and entropy \citep{kuindersma2013variable,tamar2016learning}.
Recent work in distributional reinforcement learning (RL) have also investigated estimating the cumulative distribution of returns \citep{bellemare2017distributional, dabney2020distributional} and its various statistical functionals \citep{rowland2019statistics}.
While it may even be beneficial to use all of these different metrics simultaneously to inform better decision-making, even individually estimating and bounding any performance metric, other than mean and variance, in the \textit{off-policy setting} has remained an open problem.

This raises the main question of interest: \textit{How do we develop a universal off-policy method---one that can estimate any desired performance metrics and can also provide finite-sample confidence bounds that hold simultaneously with high probability for those metrics?}

\textbf{Prior Work:}
Off-policy methods can be broadly categorized as model-based or model-free \citep{SuttonBarto2}.
%
%
Model-based methods typically require strong assumptions on the parametric model when statistical guarantees are needed.
%
Further, using model-based approaches to estimate parameters other than the mean 
can also require estimating the \textit{distribution} of rewards for \textit{every} state-action pair 
in order to obtain the complete return distribution for any policy.

By contrast, model-free methods are applicable to a wider variety of settings. 
Unfortunately, the popular technique of using \textit{importance-weighted returns} \citep{precup2000eligibility} only corrects for the \textit{mean} under the off-policy distribution.
Recent work by \citet{chandak2021hcove} provides a specialized extension to only correct for the variance.
Outside RL, works in the econometrics and causal inference literature have also considered quantile treatments \citep{donald2014estimation,wang2018quantile} and inferences on counterfactual distributions \citep{dinardo1995labor,chernozhukov2013inference,firpo2016identification}, but these methods are not developed for sequential decisions and do not provide any high-confidence bounds with guaranteed coverage. Further, they often mandate stationarity, identically distributed data, and full observability (i.e., no confounding).

Existing frequentist high-confidence bounds are not only specifically designed for either the mean or variance, but also hold only \textit{individually} \citep{thomas2015highb, jiang2020minimax, chandak2021hcove}.
Instead of frequentist intervals,  a Bayesian posterior distribution over the mean return and various statistics of that distribution can also be obtained \citep{yang2020offline}.  
We are not aware of any method that provides off-policy bounds or even estimates for \emph{any} parameter of the return, while also handling different domain settings that are crucial for RL related tasks.
Therefore, a detailed discussion of existing work is deferred to Appendix \ref{apx:related}.

\textbf{Contributions:} We take the first steps towards a \textit{\underline{un}iversal \underline{o}ff-policy estimator} (UnO) that estimates and bounds the \textit{entire distribution} of returns, and then derives estimates and simultaneous bounds for all parameters of interest.
With UnO, we make the following contributions:

\textbf{A. } 
For \textit{any} distributional parameter (mean, variance, quantiles, entropy, CVaR, CDF, etc.), we provide an off-policy method to obtain \textbf{(A.1)}  model-free estimators; \textbf{(A.2)} high-confidence bounds that have guaranteed coverage \textit{simultaneously} for all parameters and that, perhaps surprisingly, often nearly match or outperform prior bounds specifically designed for the mean and the variance; and
\textbf{(A.3)}   approximate bounds using statistical bootstrapping that can often be significantly tighter.

\textbf{B. }
The above advantages hold for \textbf{(B.1)} fully observable and partially observable (i.e., with unobserved confounders) settings,
\textbf{(B.2)} Markovian and non-Markovian settings, and \textbf{(B.3)} settings with stationary, smoothly non-stationary, and discrete distribution shifts in a policy's performance. 

\textbf{Limitations:}
\label{apx:limitations}
Our method uses importance sampling and thus \textbf{(1)} Requires knowledge of action probabilities under the behavior policy $\beta$, \textbf{(2)} Any outcome under the evaluation policy should have a sufficient probability of occurring under $\beta$, and  \textbf{(3)} Variance of our estimators scales exponentially with the horizon length  \citep{guo2017using,liu2018breaking}, which may be unavoidable in non-Markovian domains \citep{jiang2016doubly}.

\textbf{Notation: }
For brevity, we first restrict our focus to the stationary setting. In Section \ref{sec:confounding}, we discuss how to tackle non-stationarity and distribution shifts. 
A \textit{partially observable Markov decision process} (POMDP) is a tuple $(\mathcal S,  \mathcal O, \mathcal A, \mathcal P, \Omega, \mathcal R, \gamma, d_0)$, where $\mathcal S$ is the set of states, $\mathcal O$ is the set of observations,  $\mathcal A$ is the set of actions, $\mathcal P$ is the transition function, $\Omega$ is the observation function, $\mathcal R$ is the reward function, $\gamma \in [0,1]$ is the discount factor, and $d_0$ is the starting state distribution.
Although our results extend to the continuous setting, for notational ease, we consider $\mathcal S, \mathcal A, \mathcal O$, and the set of rewards to be finite. 
Since the true underlying states are only partially observable, the resulting rewards and transitions from one partially observed state to another are therefore also potentially non-Markovian \citep{singh1994learning}. 
We write $S_t, O_t, A_t$, and $R_t$ to denote random variables for state, observation, action, and reward respectively at time $t$.
Let $\mathcal D$ be a data set  $(H_i)_{i=1}^n$ collected using \textit{behavior} policies $(\beta_i)_{i=1}^n$, where each $H_i$ denotes the \textit{observed trajectory}
$( O_{0}, A_{0}, \beta(A_{0}|O_{0}), R_{0}, O_{1}, ...)$. 
Notice that an observed trajectory contains $\beta(A_{t}|O_{t})$ and does not contain the states $S_t$, for all $t$. 
Let $G_i \coloneqq \sum_{j=0}^T \gamma^j R_{j}$ be the \textit{return} of  $H_i$, where   $\forall i, \,\, G_{\min} < G_{i} < G_{\max}$ for some finite constants $G_{\min}$ and $G_{\max}$, and $T$ is a finite horizon length.
Let $G_\pi$ and $H_\pi$ be the random variables for returns and complete trajectories under any policy $\pi$, respectively.
Since the set of observations, actions, and rewards are finite, and $T$ is finite, the total number of possible trajectories is finite.
Let $\mathcal X$ be the finite set of returns corresponding to these trajectories.
Let $\mathscr H_\pi$ be the set of all possible trajectories for any policy $\pi$. 
Sometimes, to make the dependence explicit, we write $g(h)$ to denote the return of trajectory $h$.
Further, to ensure that samples in $\mathcal D$ are informative,
we make a standard assumption that any outcome under $\pi$ has sufficient probability of occurring under $\beta$ 
(see Appendix \ref{apx:ass} for further discussion of assumptions in general),
\clearpage
\begin{ass} The set $\mathcal D$ contains independent (not necessarily identically distributed) observed trajectories generated using  $(\beta_i)_{i=1}^n$, such that for some (unknown) $\varepsilon > 0$, 
$(\beta_i(a|o)<\varepsilon)\implies(\pi(a| o) = 0)$, for all $o \in \mathcal O, a \in \mathcal A,$ and $i \in \{1,2,...,n\}$.
\thlabel{ass:support}
\end{ass}

\section{Idea Summary}
For the desired universal method, instead of considering each parameter individually, we suggest estimating the entire \textit{cumulative distribution function} (CDF) of returns first:
\begin{align}
     \forall \nu \in \mathbb R, \quad\quad F_\pi(\nu) \coloneqq \Pr \Big(G_\pi \leq \nu \, \Big).
\end{align}
%
%
Any distributional parameter, $\psi(F_\pi)$, can then be estimated from the estimate of $F_\pi$. 
However, we only have off-policy data from a behavior policy $\beta$, and the typical use of importance sampling \citep{precup2000eligibility} only corrects for the mean return.
To overcome this, we propose an estimator $\hat F_n$ that uses importance sampling from the \textit{perspective of the CDF} to correct for the \textit{entire} distribution of returns.
The CDF estimate, $\hat F_n$, is then used to obtain a plug-in estimator $\psi(\hat F_n)$ for any distributional parameter $\psi(F_\pi)$.

Next, we show that this CDF-centric perspective 
provides the additional advantage that, 
if we can compute a $1-\delta$ \textit{confidence band} $\mathcal F: \mathbb R \rightarrow 2^{\mathbb R}$ such that
\begin{align}
    \Pr \Big( \forall \nu \in \mathbb R, \,\, \Pr\big(G_\pi \leq \nu\big) \in \mathcal F(\nu) \Big) \geq 1 - \delta,
\end{align}
then a $1-\delta$ upper (or lower) high-confidence bound on any parameter, $\psi(F_\pi)$, can be obtained by searching for a function $F$ that  maximizes (or minimizes) $\psi(F)$ and $ \forall \nu \in \mathbb R$ has $F(\nu) \in \mathcal F(\nu)$.

\section{UnO: Universal Off-Policy Estimator}

In the\textit{ on-policy} setting, one approach for estimating any parameter of returns, $G_\pi$, might be to first estimate its \textit{cumulative distribution} $F_\pi$  and then use that to estimate its parameter $\psi(F_\pi)$.  
However, doing this in the off-policy setting requires additional consideration as 
the \textit{entire} distribution of the observed returns needs to be adjusted to estimate $F_\pi$ since the data is collected using behavior policies that can be different from the evaluation policy $\pi$.

We begin by observing that $\forall \nu \in \mathbb R, F_\pi(\nu)$ can be expanded using the fact that the probability that the return $G_\pi$ equals $x$ is the sum of the probabilities of the trajectories $H_\pi$ whose return equals $x$,
\begin{align}
    F_\pi(\nu) 
    &= \Pr(G_\pi \leq \nu)
    = \sum_{x\in\mathcal X, x \leq \nu} \Pr(G_\pi = x) = \sum_{x\in\mathcal X, x \leq \nu} \left( \sum_{h \in \mathscr H_\pi} \Pr(H_\pi = h) \mathds{1}_{\{g(h) = x\}} \right),\;\; \label{main:eqn:1}
\end{align}
 where $\mathds{1}_{A} = 1$ if $A$ is true and 0 otherwise.
Now, observing that the indicator function can be one for at most a single value less than $\nu$ as $g(h)$ is a deterministic scalar given $h$, \eqref{main:eqn:1} can be expressed as,
\begin{align}
    F_\pi(\nu) &
    =   \sum_{h \in \mathscr H_\pi} \Pr(H_\pi = h)  \sum_{x\in\mathcal X, x \leq \nu}\mathds{1}_{\{g(h) = x\}} 
    =   \sum_{h \in \mathscr H_\pi} \Pr(H_\pi = h)\Big( \mathds{1}_{\{g(h) {\color{red}\leq} \nu\}}\Big), \label{main:eqn:2} 
\end{align}
where the red color is used to highlight changes. 
Now, from \thref{ass:support} as $\forall \beta, \,\, \mathscr H_\pi \subseteq \mathscr H_\beta$,\footnote{Results can be extended to hybrid probability measures using Radon-Nikodym derivatives.}
\begin{align}
   F_\pi(\nu) &=   \sum_{h \in {\color{red}\mathscr H_\beta}} \Pr(H_\pi = h) \Big(\mathds{1}_{\{g(h) \leq \nu\}}\Big) =   \sum_{h \in \mathscr H_\beta} \Pr(H_\beta = h) \frac{ \Pr(H_\pi = h)}{\Pr(H_\beta = h)} \Big(\mathds{1}_{\{g(h) \leq \nu\}} \Big). \label{main:eqn:4}
\end{align}

The form of $F_\pi(\nu)$ in \eqref{main:eqn:4} is beneficial as it suggests a way to not only perform off-policy corrections for one specific parameter, as in prior works \citep{precup2000eligibility,chandak2021hcove}, but for the \textit{entire cumulative distribution function} (CDF) of return $G_\pi$.
Formally, let $\rho_i \coloneqq \prod_{j=0}^{T} \frac{\pi(A_{j}| O_{j})}{\beta_i(A_{j}|O_{j})}$ denote the importance ratio for $H_i$, which is equal to $\Pr(H_\pi=h)/\Pr(H_\beta=h)$ (see Appendix \ref{apx:proofs}).

Then, based on \eqref{main:eqn:4}, 
we propose the following non-parametric and model-free estimator for $F_\pi$.
\begin{align}
    \forall \nu\in \mathbb{R}, \quad \hat F_n(\nu) \coloneqq \frac{1}{n} \sum_{i=1}^n \rho_i  \mathds{1}_{\{G_i \leq \nu\}}. 
    \label{eqn:Festimator}
\end{align}
Figure \ref{fig:IS} provides intuition for \eqref{eqn:Festimator}.
In the following theorem, we establish that this  estimator, $\hat F_n$,
is unbiased and not only pointwise consistent, but also a uniformly consistent estimator of $F_\pi$, even when the data $\mathcal D$ is collected using multiple behavior policies $(\beta_i)_{i=1}^n$.
The proof (deferred to Appendix \ref{apx:proofs}) also illustrates that by using knowledge of action probabilities under the behavior policies, no additional adjustments (e.g., front-door or backdoor \citep{pearl2009causality}) are required by $\hat F_n$ to estimate $F_\pi$, even when the domain is non-Markovian or has  partial observability (confounders).
%

\begin{figure}[t]
\begin{minipage}[c]{0.42\textwidth}
    \centering
    \includegraphics[width=0.99\textwidth]{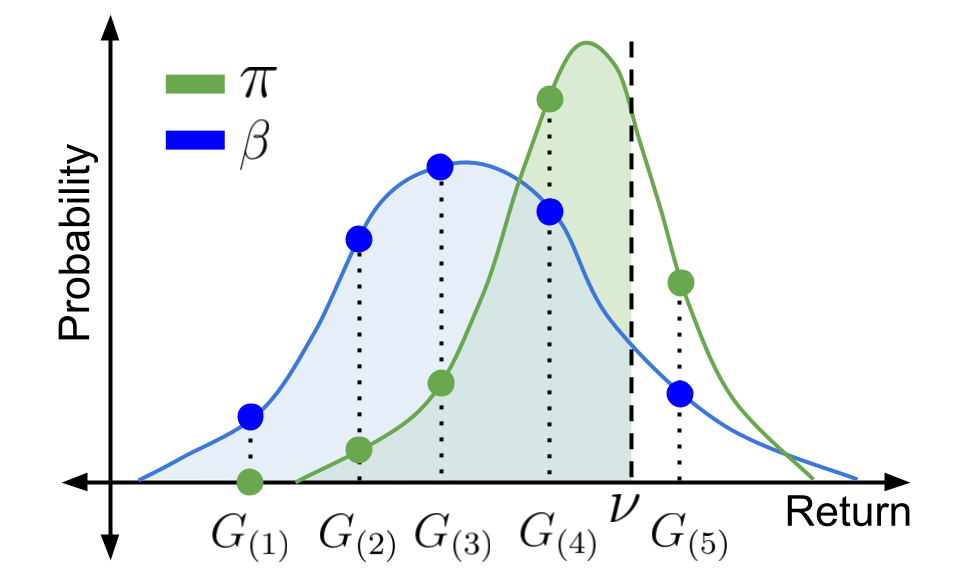}
  \end{minipage}\hfill
  \begin{minipage}[c]{0.58\textwidth}
    \caption{An illustration of return distributions for $\pi$ and $\beta$.
    The CDF at any point $\nu$ corresponds to  the area under the probability distribution up until $\nu$.
    Having order statistics $(G_{(i)})_{i=1}^5$ of samples $(G_i)_{i=1}^5$ drawn using $\beta$, \eqref{eqn:Festimator} constructs an empirical estimate of the CDF
    for $\pi$ (\textit{green} shaded region) by correcting for the probability of observing each $G_i$ using the \textit{importance-sampled counts} of $G_i \leq \nu$.
    %
    Additionally, weighted-IS (WIS) can be used as in \eqref{eqn:WISF} for a variance-reduced estimator for $F_\pi$.}
    \label{fig:IS}
  \end{minipage}
\end{figure}

\begin{thm}
\thlabel{thm:Funbiased}
Under \thref{ass:support}, $\hat F_n$ is an unbiased and uniformly consistent estimator of $F_\pi$,
\begin{align}
    \forall \nu\in \mathbb{R}, \quad \mathbb{E}_{\mathcal D}\Big[\hat F_n(\nu)\Big] &= F_\pi(\nu), &
    \underset{\nu \in \mathbb R}{\sup} \quad \Big|\hat F_n(\nu) -  F_\pi(\nu) \Big| \overset{\text{a.s.}}{\longrightarrow} 0.
\end{align}
\end{thm}
\begin{rem}
Notice that the value of $\hat F_n(\nu)$ can be more than one, even though $F_\pi(\nu)$ cannot have a value greater than one for any $\nu \in \mathbb R$.
This is an expected property of estimators based on importance sampling (IS). 
For example, the IS estimates of expected return during off-policy mean estimation can be smaller or larger than the smallest and largest possible return when $\rho > 1$. 
\thlabel{rem:geq1}
\end{rem}

Having an estimator $\hat F_n$ of $F_\pi$, any parameter $\psi(F_\pi)$ can now be estimated using $\psi(\hat F_n)$. 
However, some parameters like the mean $\mu_\pi$, variance $\sigma^2_\pi$, and entropy ${\mathcal H_\pi}$, are naturally defined using the probability distribution $\text{d}F_\pi$ instead of the cumulative distribution $F_\pi$.  
Similarly, parameters like the $\alpha$-quantile $ Q^\alpha_\pi$ and inter-quantile range (which provide tail-robust measures for the mean and deviation from the mean) and  conditional value at risk $\text{CVaR}_\pi^\alpha$ (which is a tail-sensitive measure) are defined using the inverse CDF $F_\pi^{-1}(\alpha)$.
Therefore, let $(G_{(i)})_{i=1}^n$ be the \textit{order statistics} for samples $(G_i)_{i=1}^n$ and $G_{(0)} \coloneqq G_{\min}$. 
Then, we define the off-policy estimator of the inverse CDF for all $\alpha \in [0,1]$, and the probability distribution estimator $\mathrm{d}\hat F_n$ as,
\begin{align}
    \hat F^{-1}_n(\alpha) &\coloneqq \min \Big\{g \in (G_{(i)})_{i=1}^n \Big| \hat F_n(g) \geq \alpha \Big\}, \quad & \text{d}\hat F_n(G_{(i)}) \coloneqq \hat F_n(G_{(i)}) - \hat F_n(G_{(i-1)}), \label{eqn:inverseF} 
\end{align}
where $\text{d}\hat F_n(\nu)\coloneqq 0$ if $\nu \neq G_{(i)}$ for any $i \in (1,\dotsc,n)$. 
Using \eqref{eqn:inverseF}, we now define off-policy estimators for parameters like the mean, variance, quantiles, and CVaR (see Appendix \ref{apx:UnOinverse} for more details on these).
This procedure can be generalized to any other parameter of $F_\pi$ for which a sample estimator $\psi(\hat F_n)$ can be directly created using $\hat F_n$ as a plug-in estimator for $F_\pi$. 
\begin{align}
     \mu_\pi(\hat F_n)&\coloneqq \sum_{i=1}^n \text{d}\hat F_n(G_{(i)}) G_{(i)},
    &
    \sigma^2_\pi(\hat F_n) &\coloneqq \sum_{i=1}^n \text{d} \hat F_n(G_{(i)}) \Big (G_{(i)} -  \mu_\pi(\hat F_n) \Big)^2,
    \\
        {Q}^\alpha_\pi(\hat F_n) &\coloneqq \hat F_n^{-1}(\alpha),
    & {\text{CVaR}}^\alpha_\pi(\hat F_n) &\coloneqq \frac{1}{\alpha}\sum_{i=1}^n \text{d}\hat F_n(G_{(i)}) G_{(i)} \mathds{1}_{\left\{G_{(i)} \leq {Q}^\alpha_\pi(\hat F_n)\right\}}.
\end{align}
\begin{rem}
Let $H_i$ be the observed trajectory for the $G_i$ that gets mapped to $G_{(i)}$ when computing the order statistics.
Note that $\text{d}\hat F_n(G_{(i)})$ equals $\rho_i/n$ for this $H_i$.
This implies that the estimator for the mean, $\mu_\pi(\hat F_n)$, reduces \textit{exactly} to the existing full-trajectory-based IS estimator \citep{precup2000eligibility}.
\end{rem}

Notice that many parameters and their sample estimates discussed above are nonlinear in $F_\pi$ and $\hat F_n$, respectively (the mean is one exception).
Therefore, even though $\hat F_n$ is an unbiased estimator of $F_\pi$, the sample estimator, $\psi(\hat F_n)$, may be a biased estimator of $\psi(F_\pi)$.
This is expected behavior because even in the on-policy setting it is not possible to get unbiased estimates of some parameters (e.g., standard deviation), and UnO reduces to the 
on-policy setting 
when $\pi=\beta$. 
However, perhaps surprisingly, we establish in the following section that even when $\psi(\hat F_n)$ is a biased estimator of $\psi(F_\pi)$, 
high-confidence upper and lower bounds can still be computed for both $F_\pi$ and $\psi(F_\pi)$.

\section{High-Confidence Bounds for UnO}
\label{sec:Fbounds}
Off-policy estimators are typically prone to high variance, and when the domain can be non-Markovian, the curse of horizon might be unavoidable \citep{jiang2016doubly}. For critical applications, this might be troublesome \citep{thomas2019preventing} and thus
 necessitates obtaining confidence intervals to determine how much our estimates can be trusted. 
%
Therefore, in this section, we aim to construct a set of possible CDFs  $\mathcal F: \mathbb R \rightarrow 2^{\mathbb R}$, called a \textit{confidence band}, such that the true $F_\pi(\nu)$ is within the set $\mathcal F(\nu)$ with high probability, i.e., $\Pr ( \forall \nu \in \mathbb R, \,\, F_\pi(\nu) \in \mathcal F(\nu) ) \geq 1 - \delta$, for any $\delta \in (0, 1]$.
%
%
%
Subsequently, we develop finite-sample bounds for any parameter $\psi(F_\pi)$ using $\mathcal F$.

In the on-policy setting, $\mathcal F$ can be constructed using the DKW inequality \citep{dvoretzky1956asymptotic} and its tight constants \citep{massart1990tight}. 
However, its applicability to the off-policy setting is unclear as \textbf{(a)} unlike the on-policy CDF estimate, the ``steps'' of an off-policy CDF estimate are not of equal heights, \textbf{(b)} the ``steps'' do not sum to one (see Figure \ref{fig:Fband}) and the maximum height of the steps need not be known either, and \textbf{(c)} DKW assumes samples are identically distributed, however, off-policy data $\mathcal D$ might be collected using multiple different behavior policies.
This raises the question:\textit{ How do we obtain $\mathcal F$ in the off-policy setting?}


Before constructing a confidence band $\mathcal F$, let us first focus on obtaining bounds for a single point, $F_\pi(\kappa)$.
Let $X \coloneqq \rho ( \mathds{1}_{\{G \leq \kappa\}})$. Then, from \thref{thm:Funbiased}, we have that $\mathbb{E}_{\mathcal D}[X] = F_\pi(\kappa)$.
This implies that a confidence interval for the mean of $X$ provides a confidence interval for $F_\pi(\kappa)$.
Using this observation, existing confidence intervals for the mean of a bounded random variable can be directly applied to $X$ to obtain a confidence interval for $F_\pi(\kappa)$. 
For example, \citet{thomas2015higha} present tight bounds for the mean of IS-based random variables by mitigating the variance resulting from the heavy tails associated with IS; we use their method on $\hat F_n(\kappa)$ to bound  $F_\pi (\kappa)$.
Alternatively, recent work by \citet{kuzborskij2020confident} can potentially be used with a WIS-based $F_\pi$ estimate \eqref{eqn:WISF}.

Before moving further, we introduce some additional notation.
Let $(\kappa_i)_{i=1}^K$ be any $K$ ``key points'' and let $\texttt{CI}_-(\kappa_i, \delta_i)$ and $\texttt{CI}_+(\kappa_i, \delta_i)$ be the lower and the upper confidence bounds on $F_\pi(\kappa_i)$  constructed at each key point using the  observation made in the previous paragraph, such that
\begin{align}
\forall i \in (1,...,K), \quad & \Pr\Big(\texttt{CI}_-(\kappa_i, \delta_i) \leq F_\pi(\kappa_i) \leq \texttt{CI}_+(\kappa_i, \delta_i)\Big) \geq 1 - \delta_i.
\end{align}
We now use the following observation to obtain a band, $\mathcal F$, that contains $F_\pi$ with high confidence.
Because $F_\pi$ is a CDF, it is 
necessarily monotonically non-decreasing, and so if $F_\pi(\kappa_i) \geq \texttt{CI}_-(\kappa_i, \delta_i)$ then for any $ \nu \geq \kappa_i$, $F_\pi(\nu)$ must be no less than $\texttt{CI}_-(\kappa_i, \delta_i)$.
Similarly, if $F_\pi(\kappa_i) \leq \texttt{CI}_+(\kappa_i, \delta_i)$ then for any $ \nu \leq \kappa_i$, $F_\pi(\nu)$ must also be no greater than $\texttt{CI}_+(\kappa_i, \delta_i)$.
Let $\kappa_0 \coloneqq G_{\min}$, $\kappa_{K+1} \coloneqq G_{\max}$, $\texttt{CI}_-(\kappa_0, \delta_0) \coloneqq 0$, and $\texttt{CI}_+(\kappa_{K+1}, \delta_{K+1}) \coloneqq 1$; then, as illustrated in Figure \ref{fig:Fband}, we can construct a lower function $F_-$ and an upper function $F_+$ that encapsulate $F_\pi$ with high probability,
  \begin{minipage}[c]{0.5\textwidth}
\begin{align}
     F_-(\nu) \coloneqq 
    \begin{cases}
        1 & \text{if  } \nu > G_{\max},
        \\
        \underset{\kappa_i \leq \nu}{\max} \,\,\, \texttt{CI}_-(\kappa_i,\delta_i) & \text{otherwise}.
    \end{cases}
\end{align}   
  \end{minipage}\hfill
  \begin{minipage}[c]{0.5\textwidth}
    \begin{align}
     F_+(\nu) \coloneqq 
    \begin{cases}
        0 & \text{if  } \nu < G_{\min},
        \\
        \underset{\kappa_i \geq \nu}{\min} \,\,\, \texttt{CI}_+(\kappa_i,\delta_i) & \text{otherwise}.
    \end{cases}    \label{eqn:Fband2}
\end{align}
\end{minipage}
%
%
\begin{figure}
  \begin{minipage}[c]{0.4\textwidth}
    \centering
    \includegraphics[width=0.9\textwidth]{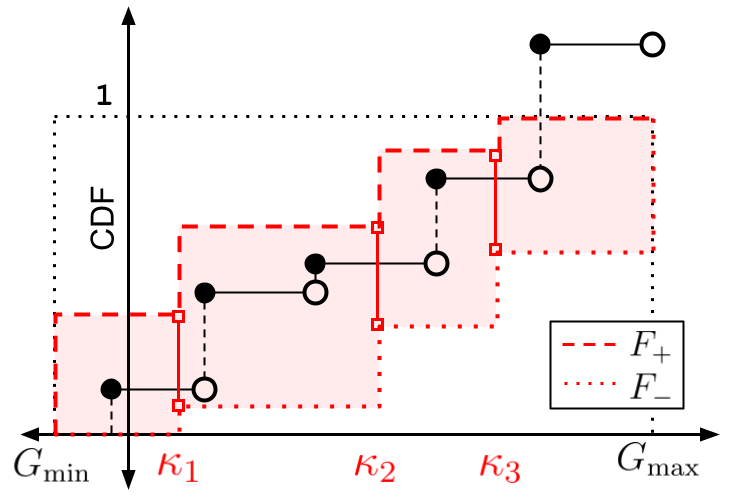}
  \end{minipage}\hfill
  \begin{minipage}[c]{0.6\textwidth}
    \caption{An illustration of $\hat F_n$ (in black) using five return samples and the confidence band $\mathcal F$ (red shaded region) computed using \eqref{eqn:Fband2} with confidence intervals (red lines) at three key points $(\kappa_i)_{i=1}^3$.
    %
    %
    Notice that the vertical ``steps'' in $\hat F_n$ can be of different heights and their total can be greater than $1$ due to importance weighting. However, since we know that $F_\pi$ is never greater than $1$,  $\mathcal F$ can be clipped at $1$.  }
    \label{fig:Fband}
  \end{minipage}
\end{figure}
\begin{thm}
Under \thref{ass:support}, for any $\delta \in (0, 1]$, if $\sum_{i=1}^K\delta_i \leq \delta$, then the confidence band defined by $F_-$ and $F_+$ provides guaranteed coverage for $F_\pi$.
That is, 
\begin{align}
    \Pr \Big(\forall \nu \in \mathbb R, \,\, F_{-}(\nu) \leq F_\pi(\nu) \leq F_+(\nu) \Big) \geq 1 - \delta.
\end{align}\thlabel{thm:Fguarantee1}
\end{thm}
\begin{rem}
  Notice that any choice of $(\kappa_i)_{i=1}^K$ results in a valid band $\mathcal F$. 
  However, $\mathcal F$ can be made tighter by optimizing over the choice of $(\kappa_i)_{i=1}^K$. In Appendix \ref{apx:UnOoptim}, we present one such method using cross-validation to minimize the area enclosed within $\mathcal F$. 
\end{rem}
Having obtained a high-confidence band for $F_\pi$, we now discuss how high-confidence bounds for any parameter $\psi(F_\pi)$ can be obtained using this band.
 Formally, with a slight overload of notation let $\mathcal F$ be the set of all possible CDFs bounded between $F_-$ and $F_+$, that is,
\begin{align}
    \mathcal F \coloneqq \Big\{F  \,\, \Big| \,\, \forall \nu \in \mathbb R, \,\,  F_-(\nu) \leq F(\nu) \leq F_+(\nu)\Big\}.
\end{align}
This band $\mathcal F$ contains many possible CDFs, one of which is $F_\pi$ with high probability.
Therefore, to get a lower or upper  bound, $\psi_-$  or $\psi_+$, on $\psi(F_\pi)$, we propose deriving a CDF $F \in \mathcal F$ that minimizes or maximizes $\psi(F)$, respectively, and we show that these contain $\psi(F_\pi)$ with high probability:
%
%
\begin{align}
    \psi_- &\coloneqq \underset{F \in \mathcal F}{\inf} \,\,\, \psi (F), \quad\quad
    \psi_+ \coloneqq \underset{F \in \mathcal F}{\sup} \,\,\, \psi (F). \label{eqn:lub}
\end{align}
\begin{thm}
Under \thref{ass:support}, for any $1-\delta$ confidence band $\mathcal F$, the confidence interval defined by $\psi_-$ and $\psi_+$ provides guaranteed coverage for $\psi(F_\pi)$.
That is,
\begin{align}
\Pr \Big(\psi_- \leq \psi(F_\pi) \leq \psi_+ \Big) \geq 1 - \delta. 
\end{align}
\end{thm}
While obtaining $\psi_-$ 
might not look straightforward, one can obtain closed-form expressions for many popular parameters of interest. 
In other cases, simple algorithms exist for computing $\psi_-$ and $\psi_+$ \citep{romano2002explicit}. 
%
Figure \ref{fig:geometry} provides geometric depictions of the closed-form expressions for some parameters. 

%
%

%
%
\begin{rem}
  Perhaps surprisingly, even though $\psi(\hat F_n)$  may be biased, 
we can obtain high-confidence bounds with guaranteed coverage on any $\psi(F_\pi)$ using the confidence band $\mathcal F$. In fact,
   confidence bounds for \emph{all} parameters computed using \eqref{eqn:lub}  
    hold \emph{simultaneously} with probability at least $1-\delta$ as they are all derived from the same confidence band, $\mathcal F$.
\end{rem}

\begin{figure*}
    \centering
    \includegraphics[width=0.32\textwidth]{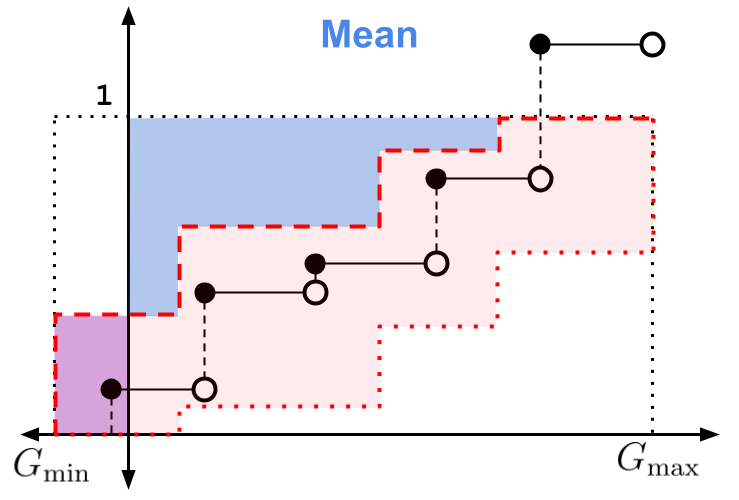}
    \includegraphics[width=0.32\textwidth]{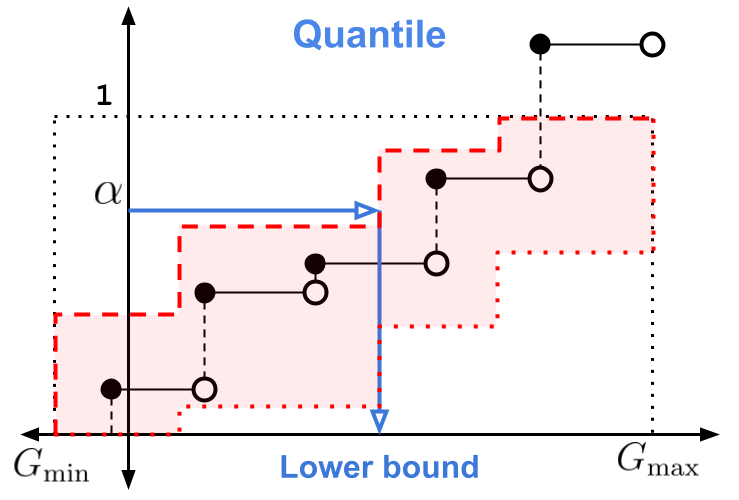}    \includegraphics[width=0.32\textwidth]{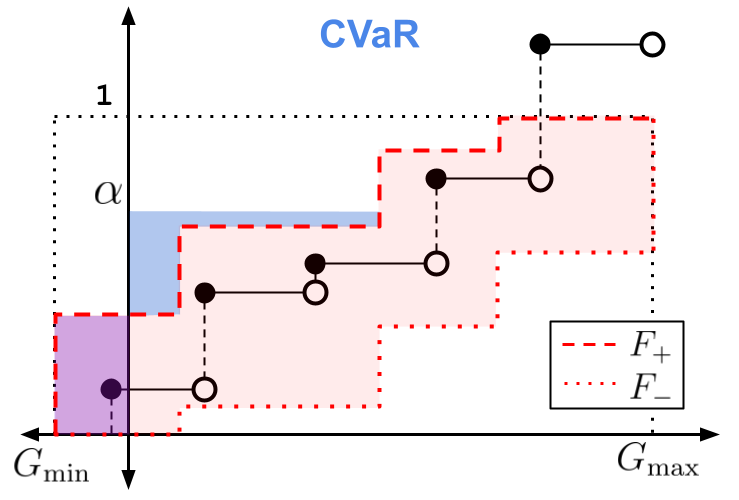}
    \caption{Given a confidence band $\mathcal F$, bounds for many parameters can be obtained using geometry.
    \textbf{(Left)} For a lower bound on the mean, we would want a CDF $F \in \mathcal F$ that assigns as high a probability as possible on lower $G$ values, and $F_+$  is the CDF which does that.
    To obtain the mean of $F_+$, we use the property that the mean of a distribution is the area above the CDF on the positive x-axis minus the area below the CDF on the negative x-axis \citep{anderson1969confidence}. 
    Hence, the mean of the distribution characterized by $F_+$ is the area of the shaded blue region minus the area of the shaded purple region, and this value is the high-confidence lower bound on the mean.
    %
    \textbf{(Middle)} Similarly, 
    within $\mathcal F$, $F_+$ characterizes the distribution with the smallest 
    $\alpha$-quantile.
    \textbf{(Right)} Building upon the lower bounds for the mean and the quantile, \citet{thomas2019concentration} showed that the lower bound for $\alpha$-CVaR can be obtained using the area of the shaded blue region minus the area of the shaded purple region, normalized by $\alpha$.
    To get the upper bounds on the mean, quantile, and CVaR, analogous arguments hold using the lower bound CDF $F_-$.
    See Appendix \ref{apx:UnOoptim} for discussions of variance, inter-quantile, entropy, and other parameters.
    %
    }
    \label{fig:geometry}
\end{figure*} 
\textbf{3.1. Statistical Bootstrapping:}
An important advantage of having constructed an off-policy estimator of any $\psi(F_\pi)$ is that it opens up the possibility of using \textit{resampling}-based methods, like statistical bootstrapping \citep{efron1994introduction}, to obtain \textit{approximate} confidence intervals for $\psi(F_\pi)$.
In particular,
we can use the \textit{bias-corrected and accelerated} (BCa) bootstrap procedure to obtain $\psi_-$ and $\psi_+$ for $\psi(F_\pi)$.
This procedure is outlined in Algorithm \ref{alg:pboot} in Appendix \ref{apx:sec:boot}.

Unlike the bounds from \eqref{eqn:lub}, BCa-based bounds do not offer guaranteed coverage and need to be computed individually for each parameter $\psi$.
%
However, they can be combined with UnO to get significantly tighter bounds with less data, albeit without guaranteed coverage.

\section{Confounding, Distributional Shifts, and Smooth Non-Stationarities}
\label{sec:confounding}
A particular advantage of  UnO is the remarkable simplicity with which the estimates and bounds for $F_\pi$ or $\psi(F_\pi)$ can be extended to account for confounding, distributional shifts, and smooth non-stationarities that are prevalent in real-world applications \citep{dulac2019challenges}.

\textbf{Confounding / Partial Observability:} 
%
Estimator $\hat F_n$ in \eqref{eqn:Festimator}  accounts for partial observability when both $\pi$ and $\beta$ have the same observation set. However, in systems like automated loan approval \citep{thomas2019preventing}, data might have been collected using a behavior policy $\beta$ dependent on sensitive attributes like race and gender that may no longer be allowable under modern laws.
%
%
This can make the available observation, $\widetilde O$, for an evaluation policy $\pi$ different from  the observations, $O$, for  $\beta$, which may also have been a partial observation of the underlying true state $S$. 

However, an advantage of many such automated systems (e.g., online recommendation, automated healthcare, robotics) is the direct availability of behavior probabilities $\beta_i(A|O)$.
In Appendix \ref{apx:proofs}, we provide generalized proofs for all the earlier results, showing that access to $\beta_i(A|O)$  allows UnO to handle various sources of confounding even when $\widetilde{O} \neq O$, without requiring any additional adjustments.
When $\beta_i(A|O)$ is not available, we allude to possible alternatives in Appendix \ref{apx:ass}.


\textbf{Distribution Shifts:} Many practical applications exhibit distribution shifts that might be discrete or abrupt. 
One example is when a medical treatment developed for one demographic is applied to another \citep{gao2020deep}. 
To tackle discrete distributional shifts, let $F^{(1)}_\pi$ and $F^{(2)}_\pi$ denote the CDFs of returns under policy $\pi$ in the first and the second domain, respectively. 
To make the problem tractable, similar to prior work on characterizing distribution shifts \citep{berger2014kolmogorov}, we assume that the Kolmogorov-Smirnov  distance between $F^{(1)}_\pi$ and $F^{(2)}_\pi$ is bounded.
%
%
%
%
\begin{ass}
There exists $\epsilon \geq 0$, such that
        $\underset{\nu \in \mathbb R}{\sup} \left | F^{(1)}_\pi(\nu) - F^{(2)}_\pi(\nu) \right| \leq \epsilon$.
    \thlabel{ass:shift}
\end{ass}
Given data $\mathcal D$ collected in the first domain, one can obtain the bounds $F^{(1)}_-$ and $F^{(1)}_+$ on $F^{(1)}_\pi$ as in Section \ref{sec:Fbounds}.
Now since $F^{(2)}_\pi$ can differ from $F^{(1)}_\pi$ by at most $\epsilon$ at any point, we propose the following  bounds for $F^{(2)}_\pi$ for all $\nu \in \mathbb{R}$ and show that they readily provide guaranteed coverage for $F^{(2)}_\pi$: 
\begin{align}
     F^{(2)}_-(\nu) &\coloneqq \max(0,F^{(1)}_-(\nu) - \epsilon),
    & F^{(2)}_+(\nu) \coloneqq \min(1,F^{(1)}_+(\nu) + \epsilon). \label{eqn:Fshiftbound}
\end{align}

\begin{thm}
Under \thref{ass:support,ass:shift},  $\forall \delta \in (0, 1]$, the confidence band defined by $F^{(2)}_-$ and $F^{(2)}_+$ provides guaranteed coverage for $F^{(2)}_\pi$.
That is, $\Pr (\forall \nu, \,\, F^{(2)}_{-}(\nu) \leq F^{(2)}_\pi(\nu) \leq F^{(2)}_+(\nu) ) \geq 1 - \delta.$
\thlabel{thm:Fguaranteeshift}
\end{thm}

\textbf{Smooth Non-stationarity:} 
The stationarity assumption is unreasonable for applications like online tutoring or recommendation systems, which must deal with drifts of students' interests or seasonal fluctuations of customers' interests \citep{thomas2017predictive,theocharous2020reinforcement}.
In the worst case, however, even a small change in the transition dynamics can result in a large fluctuation of a policy's performance and make the problem intractable.
Therefore, similar to the work of \citet{chandak2020towards}, we assume that the distribution of returns for any $\pi$ changes smoothly over the past episodes $1$ to $L$, and the $\ell$ episodes in the future.
In particular, we assume that the trend of $F^{(i)}_\pi(\nu)$ for all $\nu$ can be modeled using least-squares regression using a nonlinear basis function $\phi : \mathbb R \rightarrow \mathbb R^d$ (e.g., the Fourier basis, which is popular for modeling non-stationary trends \citep{bloomfield2004fourier}).
\begin{ass}
    For any $\nu$, $\exists w_\nu \in \mathbb{R}^d$, such that,
         $\forall i \in [1,L+\ell], \,\,\,\, F^{(i)}_\pi(\nu) = \phi(i)^\top w_\nu.$
    \thlabel{ass:ns}
\end{ass}
Estimating $F_\pi^{(L+\ell)}$ can now be seen as a time-series forecasting problem.
Formally, for any key point $\kappa$, let $\hat F_n^{(i)}(\kappa)$ be the estimated CDF using $H_i$ observed in episode $i$.
From \thref{thm:Funbiased}, we know that $\hat F_n^{(i)}(\kappa)$ is an unbiased estimator of $F^{(i)}_\pi(\kappa)$;
 therefore, $(\hat F_n^{(i)}(\kappa))_{i=1}^L$ is an unbiased estimate for the underlying time-varying sequence $(F_\pi^{(i)}(\kappa))_{i=1}^L$.
 Now, using methods from time-series literature, the trend of $(\hat F_n^{(i)}(\kappa))_{i=1}^L$ can be analyzed to forecast $F^{(L+\ell)}_\pi(\kappa)$, along with its $\texttt{CI}$s. 
 In particular, we propose using \textit{wild bootstrap} \citep{mammen1993bootstrap,davidson2008wild}, which provides \textit{approximate} $\texttt{CI}$s with finite sample error of $O(L^{-1/2})$ while also handling 
 non-normality and heteroskedasticity, which would occur when dealing with IS-based estimates resulting from different behavior polices  \citep{chandak2020towards}. 
 See Appendix \ref{apx:sec:NSboot} for more details.
Finally, using the bounds obtained using wild bootstrap at multiple key points, an entire confidence band can be obtained as discussed in Section \ref{sec:Fbounds}.

\section{Empirical Studies}

In this section, we provide empirical support for the established theoretical results for the proposed UnO estimator and high-confidence bounds.
To do so, we use the following domains: \textbf{(1)} An open source implementation \citep{xie2020deep} of the FDA-approved type-$1$ diabetes treatment simulator \citep{man2014uva}, \textbf{(2)} A stationary and a non-stationary recommender system domain, and \textbf{(3)} A continuous-state Gridworld with partial observability, where data is collected using multiple behavior policies.
Detailed description for domains and the procedures for obtaining $\pi$ and $\beta$ are provided in Appendix \ref{apx:empiricaldomain}; code is also publicly available \href{https://github.com/yashchandak/UnO}{here}.
%
%
In the following, we discuss four primary takeaway results. 

\textbf{(A) Characteristics of the UnO estimator: } Figure \ref{fig:results} reinforces the universality of UnO. 
As can be seen, UnO can accurately estimate the entire CDF and a wide range of its parameters: mean, variance, quantile, and CVaR.
%
%

\begin{figure*}[!ht]
    \centering
    \includegraphics[width=0.306\textwidth]{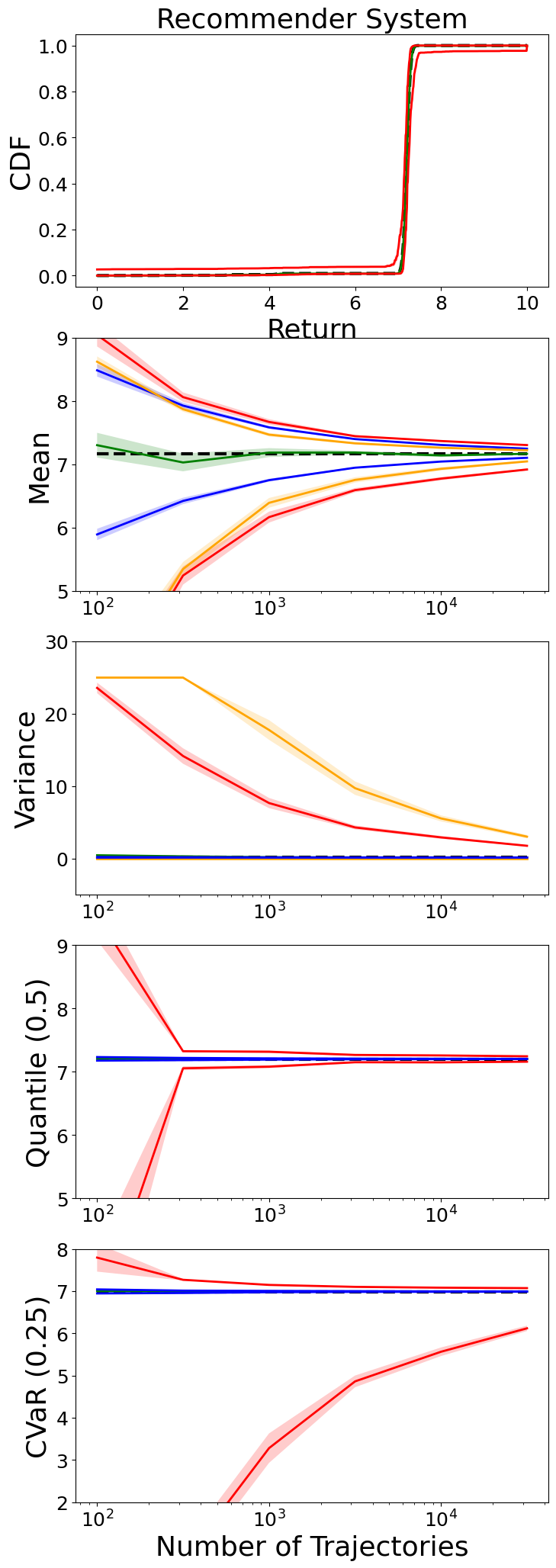}
    \includegraphics[width=0.29\textwidth]{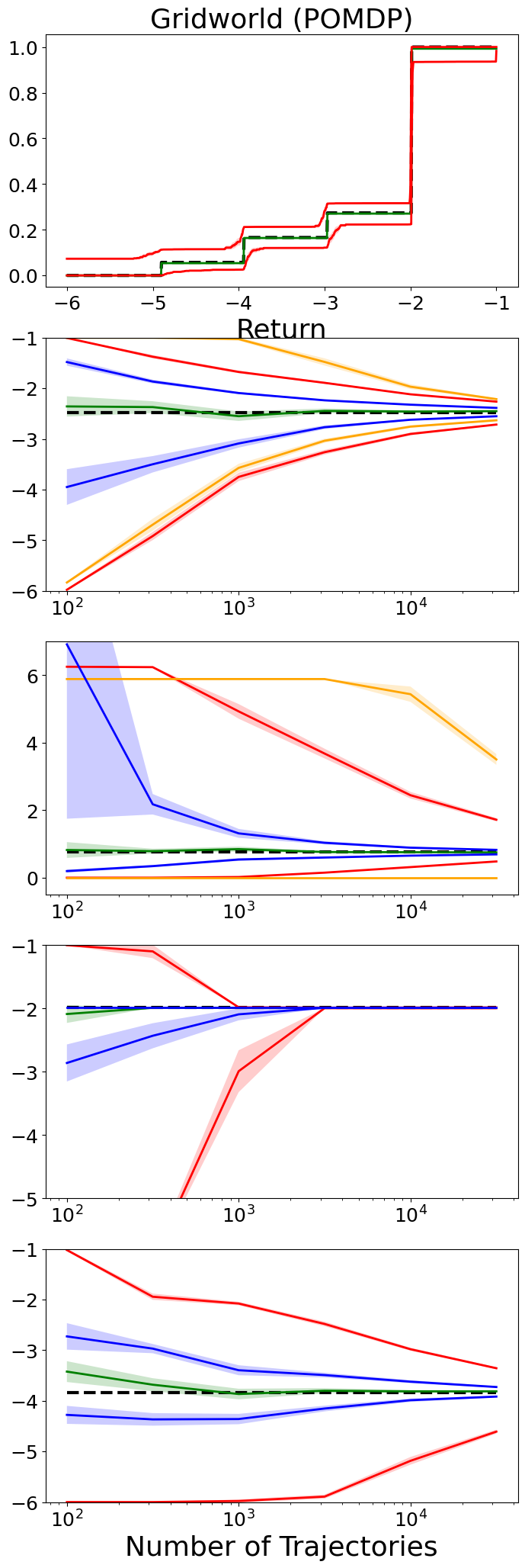}    \includegraphics[width=0.295\textwidth]{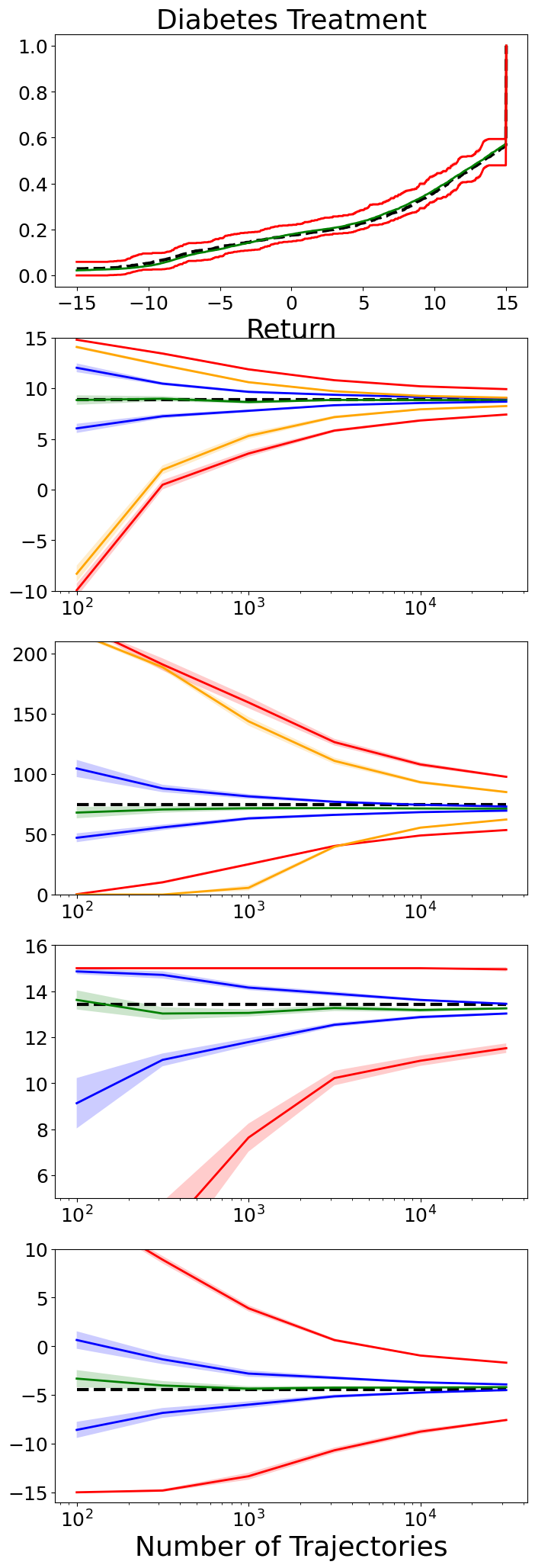}
    \\
    \includegraphics[width=0.6\textwidth]{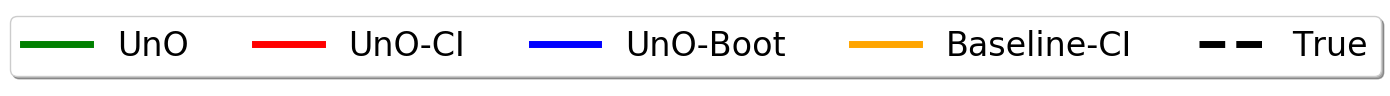}
    \caption{
    Performance trend of the proposed estimators and bounds on three domains.
    %
    The black dashed line is the true value of $F_\pi$ or $\psi(F_\pi)$, green is our UnO estimator, red is our $\texttt{CI}$-based UnO bound,  blue is the bootstrap version of our UnO bound, and yellow is the baseline bound for the mean \citep{thomas2015higha} or variance \citep{chandak2021hcove}.
    Each bound has two lines (upper and lower); however, some are not visible due to overlaps.
    The shaded regions are $\pm 2$ standard error, 
    computed using 30 trials.
    The plots in the top row are for CDFs obtained using $3\times10^{4.5}$ samples.
    The next four rows are for different parameters and share the same x-axis.
    Bounds were obtained for a failure rate $\delta = 0.05$.
    Since the UnO-Boot and Baseline-CI methods do not hold simultaneously for all the parameters, they were made to hold with failure rate of $\delta/4$ for a fair comparison (as there are 4 parameters in this plot).
    }
    \label{fig:results}
\end{figure*}

\textbf{(B) Comparison of UnO with prior work:} Recent works for bounding the mean \citep{jiang2020minimax,feng2021nonasymptotic} assume no confounding and Markovian structure. Therefore, for a fair comparison, we resort to the method of \citet{thomas2015higha} that can provide tight bounds even when the domain is non-Markovian or has confounding (partial observability). 
Perhaps surprisingly, Figure \ref{fig:results} shows that the proposed guaranteed coverage bounds, termed \textit{UnO-CI} here, can be competitive with this existing specialized bound, termed \textit{Baseline-CI} here, for the mean. 
In fact, UnO-CI can often require an order of magnitude less data compared to the specialized bounds for variance \citep{chandak2021hcove}; 
we refer readers to Appendix \ref{apx:empiricalS} for a discussion on potential reasons. 
This suggests that the universality of UnO can be beneficial even when only one specific parameter is of interest. 

\begin{figure}[t]
    \centering
    \includegraphics[width=0.305\textwidth]{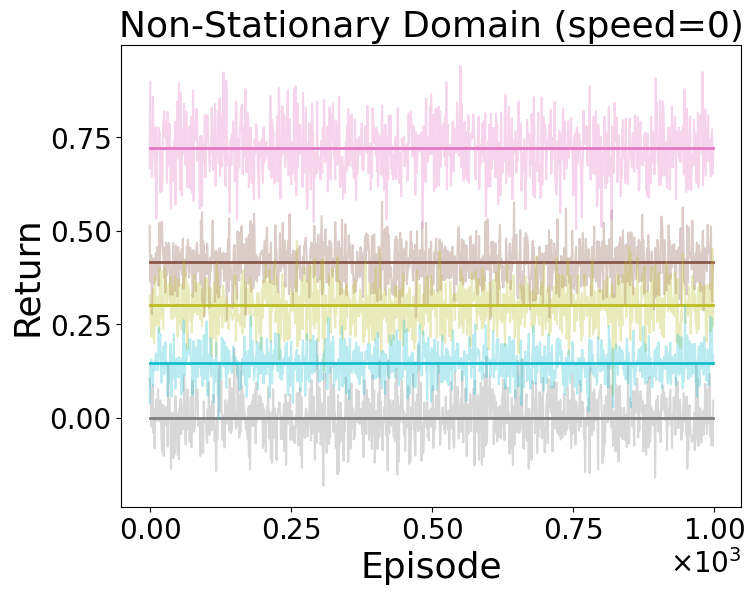} 
    \includegraphics[width=0.3\textwidth]{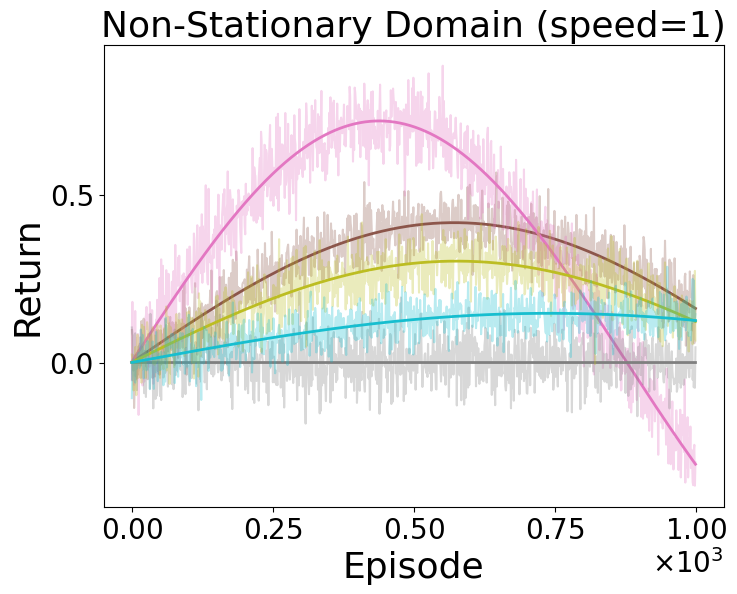} 
    \includegraphics[width=0.31\textwidth]{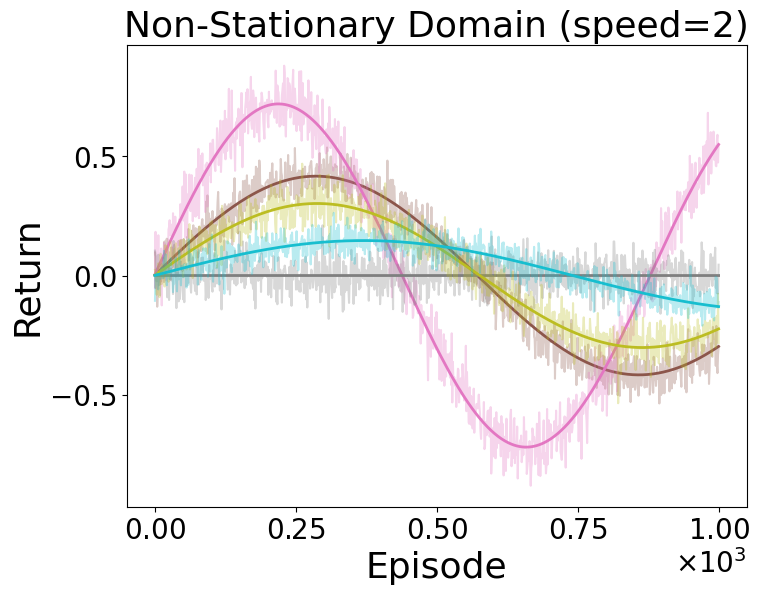}
    \\
    \includegraphics[width=0.4\textwidth]{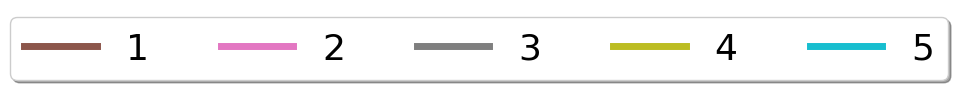}
    \\
    \includegraphics[width=0.3\textwidth]{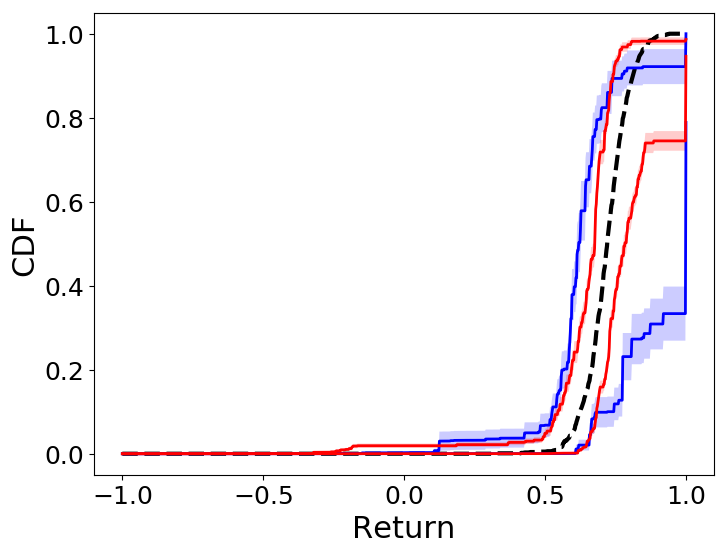}
    \includegraphics[width=0.3\textwidth]{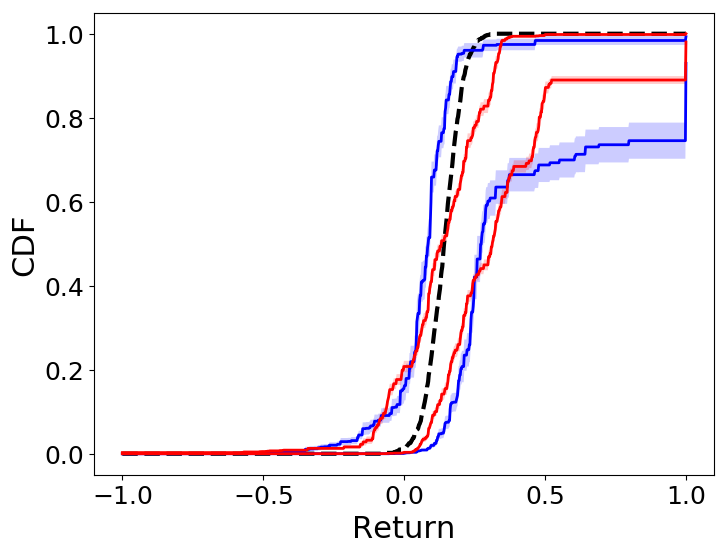}
    \includegraphics[width=0.3\textwidth]{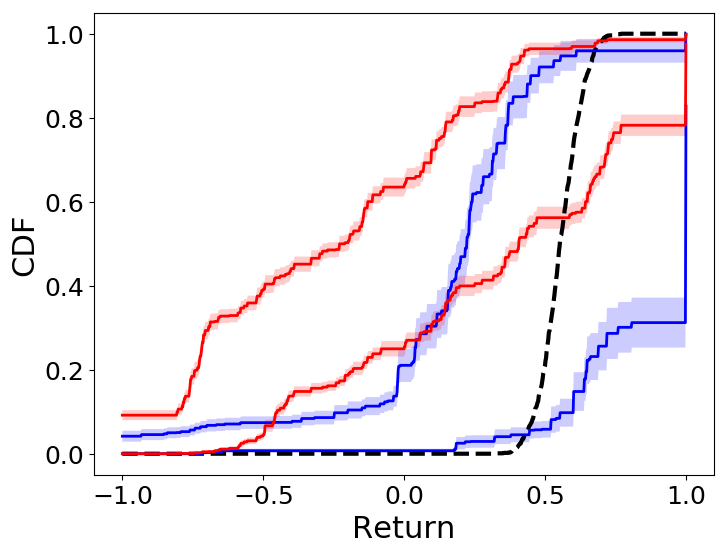}
    \\
    \includegraphics[width=0.4\textwidth]{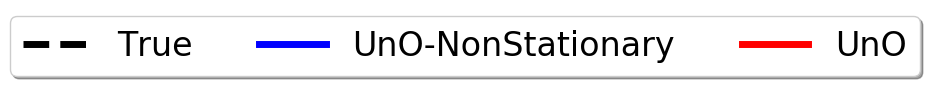}
    \caption{ \textbf{(Top row)} True rewards (unknown to the RL agent) associated with each of the five items over the past $1000$ episodes for different \textit{speeds} of non-stationarity. Speed of $0$ indicates stationary setting and higher speeds indicates greater degrees of non-stationarity.
    \textbf{(Bottom row)}. The black dashed line is the true value of the future distribution of returns under $\pi$: $F^{(L+\ell)}_\pi$, where $L=1000$ and $\ell = 1$. In red is our UnO bound that does not account for non-stationarity, and in blue is the wild-bootstrap version of our UnO bound that accounts for non-stationarity.
    The shaded region corresponds to one standard error computed using 30 trials.
    Bounds were obtained for a failure rate $\delta = 0.05$.
    \textbf{(Left column)} In the stationary setting, both the variants of UnO bounds approximately contain the true future CDF $F_\pi^{(L+\ell)}$. In this setting, the UnO method designed only for stationary settings provides a tighter bound.
    \textbf{(Middle \& Right columns)} As the domain becomes non-stationary, UnO bounds that do not account for non-stationarity fail to adequately bound the true future CDF $F_\pi^{(L+\ell)}$. 
    When the degree of non-stationarity is high, not accounting for non-stationarity can lead to significantly inaccurate bounds. 
    By comparison, UnO bounds that use wild bootstrap to tackle non-stationarity provide a more accurate bound throughout.
    As expected, when the fluctuations due to non-stationarity increase, the width of the confidence band increases as well.
    These results illustrate (a) the importance of accounting for non-stationarity, when applicable, and (b) the flexibility offered by our proposed universal off-policy estimator, UnO, to tackle such settings. 
    \vspace{-10pt}
    }
    \label{fig:nsresults}
\end{figure}

\textbf{(C) Finite-sample confidence bounds for other parameters using UnO: }
Figure \ref{fig:results} demonstrates that UnO-CI also successfully addresses the open question of  providing guaranteed coverage bounds for multiple parameters simultaneously without additional applications of the union bound. 
As expected, bounds for parameters like variance and CVaR that depend heavily on the distribution tails take more samples to shrink than bounds on other parameters (like the median [quantile($0.5$)]).
Additional discussion on the observed trends for the bounds is provided in Appendix \ref{apx:empiricalS}.

The proposed UnO-Boot bounds, as discussed in Section 3.1, are approximate and might not always hold with the specified probability.
 %
 However, they stand out by providing \textit{significantly} tighter, and thus more practicable, confidence intervals.

\textbf{(D) Results for non-stationary settings:  }
\label{apx:empiricalNS}
%
Results for this setting are presented in Figure \ref{fig:nsresults}.
As discussed earlier, online recommendation systems for tutorials, movies, advertisements and other products are ubiquitous.
However, the popular assumption of stationarity is seldom applicable to these systems.
In particular, personalizing for each user is challenging in such settings as interests of a user for different items among the recommendable products fluctuate over time.
For an example, in the context of online shopping, interests of customers can vary based on seasonality or other unknown factors.
To abstract such settings, in this domain the reward (interest of the user) associated with each item changes over time.
See Figure \ref{fig:nsresults} (top row) for visualization of the domain, for different  ``speeds'' (degrees of non-stationarity). 
%

%
%

In all the settings with different speeds, a uniformly random policy was used as a behavior policy $\beta$ to collect data for $1000$ episodes.
To test the efficacy of UnO, when the future domain can be different from the past domains, the evaluation policy was chosen to be a near-optimal policy for the future episode: $1000+1$.

 \section{Conclusion} 

We have taken the first steps towards developing a \emph{\underline{un}iversal \underline{o}ff-policy estimator} (UnO), closing
 the open question of whether it is possible to estimate and provide finite-sample bounds (that hold with high probability) for \textit{any} parameter of the return distribution in the \textit{off-policy} setting, with minimal assumptions on the domain. 
Now, without being restricted to the most common and
basic parameters, researchers and practitioners can  
fully characterize the (potentially dangerous or costly) behavior of a policy without having to deploy it.
 There are many new questions regarding how UnO can be improved for policy \textit{evaluation} by further reducing data requirements or weakening assumptions. 
 Using UnO for policy \textit{improvement} also remains an interesting future direction.
 Subsequent to this work, 
 \citet{huang2021offpolicy} showed how models can be used to obtain UnO-style doubly robust estimators along with its convergence rates in the contextual bandit setting. This allows their method to also provide finite-sample uniform CDF bounds for a broad class of Lipschitz risk functionals.

\section{Acknowledgements}

We thank Shiv Shankar, Scott Jordan, Wes Cowley, and Nan Jiang for the feedback, corrections, and other contributions to this work. We would also like to thank Bo Liu, Akshay Krishnamurthy, Marc Bellemare, Ronald Parr, Josiah Hannah, Sergey Levine, Jared Yeager, and the anonymous reviewers for their feedback on this work.


Research reported in this paper was sponsored in part by a gift from Adobe, NSF award \#2018372, and the DEVCOM Army Research Laboratory under Cooperative Agreement W911NF-17-2-0196 (ARL IoBT CRA).  Research in this paper is also supported in part by NSF (IIS-1724157, IIS-1638107, IIS-1749204, IIS-1925082), ONR (N00014-18-2243), AFOSR (FA9550-20-1-0077), and ARO (78372-CS, W911NF-19-2-0333). The views and conclusions contained in this document are those of the authors and should not be interpreted as representing the official policies, either expressed or implied, of the Army Research Laboratory or the U.S.~Government. The U.S.~Government is authorized to reproduce and distribute reprints for Government purposes notwithstanding any copyright notation herein.




 \typeout{}
\bibliography{neurips_bib}
\bibliographystyle{abbrvnat}


\clearpage
\include{appendix}

\end{document}

%% file: header_neurips.tex
\usepackage{amsmath}
\usepackage{amsthm}
\usepackage{amssymb}
\usepackage{dsfont}
\usepackage{bbm}
\usepackage{mathtools}
\usepackage{mathrsfs}
\usepackage{theoremref}
\mathtoolsset{showonlyrefs}
\usepackage{multicol}

\newtheorem{thm}{Theorem}[]

\newtheorem{prop}{Property}[]
\newtheorem{ass}{Assumption}[]

\newtheorem{rem}{Remark}[]

\usepackage{courier} 

\usepackage{lipsum} 

\usepackage{graphicx}
\usepackage[space]{grffile} 

\usepackage[numbers]{natbib}

\usepackage{hyperref}

\usepackage{graphicx}
\usepackage{wrapfig}
\usepackage{xcolor}
\definecolor{dark-red}{rgb}{0.4,0.15,0.15}
\definecolor{dark-blue}{rgb}{0,0,0.7}
\hypersetup{
    colorlinks, linkcolor={dark-blue},
    citecolor={dark-blue}, urlcolor={dark-blue}
}

\DeclareMathOperator*{\argmin}{arg\,min}
\DeclareMathOperator*{\argmax}{arg\,max}

\usepackage{algpseudocode}
\usepackage[vlined,linesnumbered,ruled,algo2e]{algorithm2e}
\SetKwProg{Fn}{Function}{}{}
\let\oldnl\nl
\newcommand{\nonl}{\renewcommand{\nl}{\let\nl\oldnl}}

\usepackage{multicol}
\usepackage{enumitem}
\usepackage[utf8]{inputenc} 
\usepackage[T1]{fontenc}    
\usepackage{hyperref}       
\usepackage{url}            
\usepackage{booktabs}       
\usepackage{nicefrac}       

\usepackage[bottom]{footmisc}

%% file: appendix.tex
\onecolumn
\setcounter{thm}{0}
\setcounter{ass}{0}

\appendix

\section{Notation}

\begin{table}[h]
    \centering
    \begin{tabular}{c|l c}
    \hline \\
    Symbol & Meaning \\
    \hline \\
    $\mathcal D$ & Data set of the observed trajectories\\
    $n$ & Total number of observed trajectories in $\mathcal D$\\
    $\pi$ & Evaluation policy\\
    $\beta_i$ & Behavior policy for the $i^{\text{th}}$ trajectory\\
    $\rho_i$ & Importance ratio for the observed trajectory $H_i$\\
     $\mathcal S$    & State set\\
     $\mathcal O$, 
     $\widetilde{\mathcal O}$     & Observation set for the behavior policy and the evaluation policy, respectively\\
    $\mathcal A$     & Action set\\
    $\mathcal P$ & Transition dynamics, $\mathcal P : \mathcal S \times \mathcal A \rightarrow \Delta(\mathcal S)$\\
    $\mathcal R$ & Reward function, $\mathcal R: \mathcal S \times \mathcal A \rightarrow \Delta(\mathbb R)$\\
    $\Omega$ & Observation function for behavior policy, $\Omega : \mathcal S \rightarrow \Delta(\mathcal O)$\\
    $\Omega_2$ & Observation function for the evaluation policy, $\Omega_2: \mathcal S \times \mathcal O \rightarrow \Delta(\widetilde {\mathcal O})$\\
    $\gamma$ & Discounting factor\\
    $d_0$ & Starting state distribution\\
    $T$ & Finite horizon length\\
    $H_i$, $H_\pi$ & $i^\text{th}$ observed trajectory in the dataset and complete trajectory under policy $\pi$, respectively\\
    $G_i$, $G_\pi$ & Return observed in the $i^\text{th}$ trajectory in the dataset and return under any policy $\pi$, respectively\\
    $G_{\min}$, $G_{\max}$  & Minimum and maximum value of a return, respectively\\
    $F_\pi, \mathrm{d}F_\pi$ & True CDF  of returns under policy $\pi$ and its associated probability distribution, respectively\\
    $\hat F_n$, $\bar F_n$ & Off-policy CDF estimator and weighted off-policy CDF estimator using $n$ samples, respectively\\
    $F_-, F_+$ & Lower and upper bound on the CDF\\
    $\mathcal F$ & The set of all CDFs between the upper bound and the lower bounds\\
    $\kappa_i, K$ & $i^\text{th}$ key point and total number of key points, respectively\\  
    $\alpha$ & Value for defining inverse CDF-based statistics \\
    $\psi$ & Generic functional for a distributional parameter/statistic\\
    $\psi_-, \psi_+$ & Lower and upper bounds for $\psi(F_\pi)$ \\
    $\delta$ & Failure rate for the bounds\\
    $\mathcal D_\text{eval}, \mathcal D_\text{train}$ & Evaluation and training split of the dataset $\mathcal D$\\
    $\texttt{CI}_-, \texttt{CI}_+$ & Lower and upper confidence bounds for a given random variable\\
    $\theta$ & Parameters that are used to construct $\mathcal F$\\
    $\mathscr A$ & Euclidean area enclosed within $\mathcal F$\\
    $X_i^*$ & $i^\text{th}$ bootstrap resampled value for any random variable $X$\\
    $\varepsilon$, $\epsilon$ & Some small value in \thref{ass:support} and \thref{ass:shift}, respectively\\
    $w_\nu$, $\phi$  & Regression weights and basis function for the assumption on smooth non-stationarity\\
    $L, \ell$ & Number of past and future episodes being considered in the smooth non-stationary setting
    
    \end{tabular}
    \caption{List of symbols used in the main paper and their associated meanings.}
    \label{tab:my_label}
\end{table}

\section{Broader Impact}
\label{apx:broader}

While our estimators and bounds are both theoretically sound and intuitively simple, it is important for a broader audience to understand the limitations of our method, assumptions being made, and what can be done when these assumptions do not hold.
Understanding these assumptions can also help in mitigating any undesired biases in applications built around UnO and can thus avoid any potential negative societal impacts. 
In the following, we briefly allude  to possible alternatives when the required  assumptions are violated.
%
%
%

\subsection{Discussion of Assumptions and Requirements of UnO}
\label{apx:ass}

\paragraph{Knowledge of Subset Support:} Through \thref{ass:support}, UnO requires that all the behavior policies $(\beta_i)_{i=1}^n$ have sufficient support for actions that have non-zero probability under $\pi$.
Particularly, it requires that the $\beta(a|o)$ is bounded below by (an unknown) $\epsilon$ when $\pi(a|o) > 0$.
This ensures that importance ratios are bounded and thus simplifies analysis for UnO's consistency results and constructing confidence intervals.
This assumption is common both in the off-policy literature \citep{kallus2020double, xie2019towards, yang2020offline} and in real applications \citep{theocharous2015personalized}

The above assumption is also equivalent to assuming bounded exponentiated-Renyi-divergence  (for $\alpha=\infty$) between the probability distributions of trajectories under the behavior and the evaluation policies  \citep{metelli2020importance}.
As the UnO's bound for the CDF uses CIs for the mean as a sub-routine, the above assumption can be relaxed by using CIs for the mean that depend on Renyi-divergence for other values of  $\alpha$ \citep{metelli2020importance}.
Similarly, consistency results for UnO rely upon finite variance, which can also be achieved by instead assuming that the Renyi-divergence is bounded for $\alpha=2$.

Alternatively, \thref{ass:support} can be relaxed to only absolute continuity by using methods that provide valid CIs for the mean by clipping the  importance weights. (See the work by \citet[Theorem 1]{thomas2015higha} for removal of the upper bound on the importance weights when lower-bounding the mean, and the work by \citet[Theorem 5]{chandak2021hcove} for removal of the upper bound on the importance weights when upper-bounding the mean).
Furthermore, prior work has also shown how even the assumption of absolute continuity can in some cases be removed (See discussion around Eqn 8 in the appendix of the work by \citet{thomas2015higha}).
 If the supports for the behavior and the evaluation policies are unequal, \citet{thomas2017importance} also present a technique to reduce variance resulting from IS.

Further, WIS might also be helpful in relaxing the assumptions on the IS ratios.  Specifically, WIS-based mean bounds \citep{kuzborskij2020confident} can also be used along with the WIS-based UnO estimator \eqref{eqn:WISF} to get a valid confidence band for the entire CDF.

Using multi-importance sampling (MIS), the subset support requirement for \textit{all}  $(\beta_i)_{i=1}^n$ can be relaxed to the requirement that the \textit{union of supports} under the behavior policies $(\beta_i)_{i=1}^n$ has sufficient support \citep{veach1995optimally,papini2019optimistic,metelli2020importance}.
MIS can also help in substantially reducing variance.
However, this relaxation requires an alternate assumption that a complete knowledge of all the behavior policies $(\beta_i)_{i=1}^n$, not just the probabilities of the action executed using them, is available.  

\paragraph{Knowledge of Action Probabilities under Behavior Policies $(\beta_i)_{i=1}^n$:}
 UnO requires access to the probability $\beta(a|o)$ (only the scalar probability value and not the entire policy $\beta$) of the actions available in the data set, $\mathcal D$, to compute the importance sampling ratios in \eqref{eqn:Festimator}.
Access to the probability $\beta(a|o)$ is often available when  $\mathcal D$ is collected using an automated policy; however, it might not be available in some cases, such as when decisions were previously made by humans. 

When the probability $\beta(a|o)$ is not available, one natural alternative is to estimate it from the data and use this estimate of $\beta(a|o)$ in the denominator of the importance ratios.
This technique is also known as regression importance sampling (RIS) and is known to provide biased but consistent estimates for the mean \citep{hanna2019importance, pavse2020reducing} in the Markov decision process setting (MDP) setting.
For UnO, $\hat F_n(\nu)$ is analogous to mean estimation of $X \coloneqq \rho \big(\mathds{1}_{\{G\leq \nu\}} \big)$, for any $\nu$.
Therefore, the findings of RIS can be directly extended to UnO in the MDP setting, where $\widetilde {O} = O = S$.
In the following, we provide a high-level discussion for the setting when $\beta(a|o)$ is \textit{not} available and the states are partially observed,
%
\begin{itemize}[leftmargin=*]
    \item \textbf{Partial observability with $\widetilde O = O$:} In this setting, as $\beta(a|o) = \beta(a|\tilde o)$, one can use density estimation on the available data, $\mathcal D$, to construct an estimator $\hat \beta(a|o)$ of $\Pr(a|\tilde o) = \beta(a|\tilde o)$ and use RIS to get a biased but consistent estimator for $F_\pi$.
    Here, bias results from the estimation error in $\hat \beta(a|o)$ but consistency follows as the true $\beta(a|o)$ can be recovered in the limit when $n \rightarrow \infty$. 

    In context of UnO, using $\hat \beta(a|o)$ instead of $\beta(a|o)$ violates the unbiased condition for $\hat F_n$, which was necessary to obtain the $\texttt{CI}$s and construct $\mathcal F$. 
    Therefore, high-confidence bounds with guaranteed coverage cannot be obtained using UnO in this setting.
    However, point estimates and approximate bootstrap bounds can still be obtained.
    \item\textbf{Partial observability with $\widetilde O \neq O$:} In this setting, using RIS will produce neither an unbiased nor a consistent estimator for $F_\pi$.
    As $\mathcal D$ only has $\tilde o$ and not $o$, at best it is only possible to estimate $\Pr(a|\tilde o) = \sum_{x \in \mathcal O} \beta(a|x)\Pr(x|\tilde o)$ through density estimation using data $\mathcal D$.  
    However, in general, since $\beta(a|o) = \Pr(a|o) \neq \Pr(a|\tilde o)$ we cannot even consistently estimate the denominator for importance sampling unless some other stronger assumptions are made.
    See work by \citet{namkoong2020off,tennenholtz2020off,bennett2020off} and \citet{kallus2020confounding} for possible alternative assumptions and approaches to tackle this setting.
\end{itemize}

\paragraph{Knowledge of $G_{\min}, G_{\max}$: } To construct the CDF band $\mathcal F$, UnO requires knowledge of $G_{\min}$ and $G_{\max}$ in \eqref{eqn:Fband2}.
Notice from Figure \ref{fig:Fband} that knowing $G_{\max}$ helps in clipping the \textit{\underline{l}ower \underline{b}ound} for the \textit{\underline{u}pper \underline{t}ail} (LBUT) of $\mathcal F$, which otherwise would have extended to $+\infty$.
Similarly, knowing $G_{\min}$ helps in clipping the \textit{{\underline{u}pper \underline{b}ound}} for the \textit{{\underline{l}ower \underline{t}ail}} (UBLT) of $\mathcal F$, which otherwise would have extended to $-\infty$.

Typically, even if $G_{\min}$ or $G_{\max}$ is not known, they can be obtained as $R_{\min}/(1-\gamma)$ or $R_{\max}/(1-\gamma)$, respectively, where $R_{\min}$ and $R_{\max}$ are known finite lower and upper bounds for any individual reward. 
Otherwise, knowledge of $G_{\min}$ or $G_{\max}$ can be relaxed if the desired bound on $\psi$ does not depend on UBLT or LBUT, respectively.
For example, observe from Figure \ref{fig:geometry} that
    (a) The lower bound for the mean or quantile does not depend on LBUT. Analogously, if only an upper bound for the mean or quantile is required, then UBLT is not needed. 
    %
    (b) The lower bound on CVaR depends on UBLT, however, (for small values of $\alpha$) the upper bound on CVaR neither depends on LBUT nor UBLT.
    %
    (c) For an upper bound on variance, both LBUT and UBLT are required. However, for the variance's lower bound, neither LBUT nor UBLT are required. See Figure \ref{apx:fig:geometry} for intuition.

\paragraph{Knowledge of Function Class $\phi$: } For the smoothly non-stationary setting, through \thref{ass:ns}, UnO requires access to the basis functions $\phi$ that can be used with least-squares regression to analyze the trend in the distributions of returns $(F_\pi^{(i)}(\nu))_{i=1}^L$ for any $\nu \in \mathbb R$.
In practice, one can use  sufficiently flexible basis functions to model time-series trends (e.g., Fourier basis \citep{bloomfield2004fourier}).
To avoid overfitting or underfitting, one could also use goodness-of-fit tests to select the functional class $\phi$ for the trend \citep{chen2003empirical}.

\paragraph{Knowledge of Bound $\epsilon$ on the Distribution Shift: } Unlike the smoothly non-stationary setting, if the underlying shift can be discrete and arbitrary, prior data may not contain any useful information towards characterizing the shift. Therefore, avoiding domain knowledge may be inevitable when setting the value for $\epsilon$ unless some other stronger assumptions are made.

  
 %

\section{Extended Discussion on Related Work}
\label{apx:related}

In the on-policy RL literature, parameters other than the mean have also been explored \citep{jaquette1973markov,sobel1982variance,chung1987discounted, white1988mean,dearden1998bayesian,lattimore2012pac,azar2013minimax}, and recent distributional RL methods extend this direction by estimating the entire distribution of returns \citep{morimura2010nonparametric,morimura2012parametric,bellemare2017distributional,dabney2018implicit,dabney2018distributional,dabney2020distributional,rowland2018analysis}.
Our work builds upon many of these ideas and extends them to the off-policy setting.

In the off-policy RL setup, there is a large body of literature that tackles the off-policy mean estimation problem \citep{precup2000eligibility,SuttonBarto2}.
Some works also aim at providing high-confidence off-policy mean estimation  using concentration inequalities \citep{thomas2015higha,kuzborskij2020confident} or bootstrapping \citep{thomas2015highb, hanna2017bootstrapping,kostrikov2020statistical}.
Several recent approaches build upon a dual perspective for dynamic programming \citep{puterman1990markov,wang2007dual,nachum2020reinforcement} for both estimating and bounding the mean \citep{liu2018breaking, xie2019towards,jiang2020minimax,uehara2020minimax,dai2020coindice,feng2021nonasymptotic}. 
However, these methods are restricted to domains with Markovian dynamics and full observability. 
Some works have also focused on estimating the mean return in 
the setting where states are partially observed \citep{namkoong2020off,tennenholtz2020off,kallus2020confounding} or when there is non-stationarity \citep{chandak2020towards,chandak2020optimizing,  khetarpal2020towards, padakandla2020survey}.
Recent work by \citet{chandak2021hcove} also looks at (high-confidence) off-policy variance estimation. 
Our work extends these research directions by tackling these settings simultaneously, while also 
providing a general procedure to estimate and 
obtain high-confidence bounds for \textit{any} parameter of the distribution of returns.
Particularly, UnO is a single, unified, and universal procedure that can be used to mitigate the complexity associated with estimating different parameters for different domain settings.

%
%
%

A popular RL method that has similar name to UnO is the \textit{Universal value function approximator} (UVFA) by  \citet{schaul2015universal}. However, UVFA is fundamentally different from UnO: UVFA estimates \textit{expected} return $\mathbb{E}[G_\pi]$ from a state given any desired goal.
By comparison, UnO estimates any parameter of the return $G_\pi$ for a single ``goal''.
Recent work by \citet{harb2020policy} and \citet{faccio2020parameter} propose using supervised learning to estimate parametric models that can map a \textit{representation} of a policy $\pi$ to the corresponding distribution of $G_\pi$.
By training over a given distribution of policies, new policies in the test set can be evaluated without using new data.
By comparison, UnO does not requires any parametric assumptions or any train-test distribution.
Further, UnO also provides high-confidence bounds for all the parameters of the return distribution. 

\section{Proofs for Theoretical Results}
\label{apx:proofs}

The main results in this paper are for the setting where both the evaluation and the behavior policies have the same observation set.
In the following, we present generalized results where the available observations, $\widetilde{O}$, for the evaluation policy can be different from the behavior policy's observations, $O$.
Further, for notational ease, in the main paper we had focused only on finite sets.
In the following, we present a more general setting where states, actions, observations, and rewards are all continuous.
Let $\Omega_2: \mathcal S \times \mathcal O \rightarrow \Delta(\widetilde {\mathcal O})$ be the distribution over $\widetilde{\mathcal O}$, conditioned on state $s \in \mathcal S$ and observation $o \in \mathcal O$, which determines how the observations $\widetilde O$ are generated.

Let $\mathcal D = (H_i)_{i=1}^n$ be the available observed trajectories, where each $H$ contains $(\widetilde O_{0}, A_{0}, \beta(A_{0}|O_{0}), R_{0}, \widetilde O_{1}, ...)$. 
Note that when the random variables  $\widetilde{O} =  O = S$, we recover a standard fully observable MDP setting.
By comparison, $H_\pi$ is the random variable corresponding to the complete trajectory $(S_0, O_0, \widetilde O_{0}, A_{0},  R_{0}, S_1,  O_1, \widetilde O_{1}, ...)$ under any policy $\pi$. 
Of course, $H_\pi$ is unknown.
To make the dependence between a trajectory $h \in \mathscr H_\pi$ and its associated return $G$ and importance ratios $\rho$ explicit, we use the shorthand $g(h)$ and $\rho(h)$ to denote the return and importance ratios for the full trajectory $h$, respectively. 
To tackle this generalized setting, we also generalize the support assumption introduced earlier,

\begin{ass}
 The set $\mathcal D$ contains independent (not necessarily identically distributed) observed trajectories generated using  $(\beta_i)_{i=1}^n$, such that for some (unknown) $\varepsilon > 0$,  
$(\beta(a|o)<\varepsilon)\implies(\pi(a|\tilde o) = 0)$, for all $s \in \mathcal S, o \in  \text{supp}(\Omega(s)), \tilde o \in \text{supp}(\Omega_2(s,o)), a \in \mathcal A,$ and $i \in \{1,\dotsc,n\}$.
\end{ass}

\begin{thm}
Under \thref{ass:support}, $\hat F_n$ is an unbiased and uniformly consistent estimator of $F_\pi$. That is,
\begin{align}
    \forall \nu\in \mathbb{R}, \quad \mathbb{E}_{\mathcal D}\Big[\hat F_n(\nu)\Big] &= F_\pi(\nu), &&
    \underset{\nu \in \mathbb R}{\sup} \quad \Big|\hat F_n(\nu) -  F_\pi(\nu) \Big| \overset{\text{a.s.}}{\longrightarrow} 0.
\end{align}
\end{thm}
\begin{proof}
This theorem has two results: unbiasedness and consistency of $\hat F_n$.
Therefore, we break the proof into two parts. 

\paragraph{Part 1 (Unbiasedness). }
We begin by expanding $F_\pi$ for any $\nu \in \mathbb R$ using the definition of the CDF.
\begin{align}
    F_\pi(\nu) 
    &= \Pr(G_\pi \leq \nu)
    = \int_{-\infty}^{\nu} p(G_\pi = x) \mathrm{d}x = \int_{-\infty}^\nu \left( \int_{\mathscr H_\pi} p(H_\pi = h) \mathds{1}_{\{g(h) = x\}}\mathrm{d}h \right)\mathrm{d}x, \quad \label{eqn:1}
\end{align}
where we used the fact that the probability density of the return $G_\pi$ being $x$ is the integral of the probability densities of the trajectories $h$ whose return equals $x$.
Therefore, as the integrands in \eqref{eqn:1} are finite and non-negative measurable functions, using Tonelli's theorem for interchanging the integrals, \eqref{eqn:1} can be expressed as,
\begin{align}
    F_\pi(\nu) &=   \int_{\mathscr H_\pi} p(H_\pi = h) \left( \int_{-\infty}^{\nu}\mathds{1}_{\{g(h) = x\}} \mathrm{d}x \right) \mathrm{d}h 
    =   \int_{\mathscr H_\pi} p(H_\pi = h)\Big( \mathds{1}_{\{g(h) {\color{red}{\leq}} \nu\}}\Big) \mathrm{d}h, \label{eqn:2} 
\end{align}
where the last term follows because the output of $g(h)$ is a deterministic scalar given $h$ and thus the indicator function can be one for at most a single value less than $\nu$, and where the red color is used to highlight changes. 
Next, using  \thref{ass:support} to change the support of the distribution in \eqref{eqn:2} and using importance weights we obtain, 
%
\begin{align}
   F_\pi(\nu) &=   \int_{\color{red}{\mathscr H_\beta}} p(H_\pi = h)  \Big(\mathds{1}_{\{g(h) \leq \nu\}}\Big)\mathrm{d}h = 
 \int_{\mathscr H_\beta} p(H_\beta = h) \frac{ p(H_\pi = h)}{p(H_\beta = h)} \Big(\mathds{1}_{\{g(h) \leq \nu\}} \Big)\mathrm{d}h.\;\; \label{eqn:4}
\end{align}
To simplify \eqref{eqn:4}, we recursively use the fact that $p(X, Y) = p(X)p(Y|X)$ and note that under a given policy $\pi$ the probability density of a trajectory with partial observations and non-Markovian structure is
\begin{align}
    p(H_\pi = h) =& p(s_0)p(o_0|s_0) p(\tilde o_0 | o_0, s_0) p(a_0 | s_0, o_0, \tilde o_0; \pi)
    \\
    & \times \prod_{i=0}^{T-1} \Bigg(p(r_i | h_i) p(s_{i+1}|h_i) p(o_{i+1} | s_{i+1}, h_i) p(\tilde o_{i+1} | s_{i+1}, o_{i+1}, h_i)
    \\
    &\quad\quad \times  p(a_{i+1} |s_{i+1}, o_{i+1}, \tilde o_{i+1}, h_i; \pi) \Bigg)p(r_T | h_T), \label{eqn:8}
\end{align}
where conditioning on $\pi$ emphasizes that each action is sampled using $\pi$, and $h_i$ represents the trajectory of all the states, partial observations, and actions up to time step $i$.
Therefore, using \eqref{eqn:8}, the ratio between $p(H_\pi = h)$ and $p(H_\beta = h)$ can be written as,
\begin{align}
    \frac{ p(H_\pi = h)}{p(H_\beta = h)} &= \frac{p(a_0 | s_0, o_0, \tilde o_0; \pi)}{p(a_0 | s_0, o_0, \tilde o_0; \beta)}
    \prod_{i=0}^{T-1} \frac{p(a_{i+1} |s_{i+1}, o_{i+1}, \tilde o_{i+1}, h_i; \pi)}{p(a_{i+1} |s_{i+1}, o_{i+1}, \tilde o_{i+1}, h_i; \beta)} \\
    &= 
    \prod_{i=0}^T \frac{\pi(a_i|\widetilde o_i)}{\beta(a_i| o_i)}
    \\
    &= \rho(h). \label{eqn:5}
\end{align}

Combining \eqref{eqn:4} and \eqref{eqn:5},
\begin{align}
    F_\pi(\nu) &=  \int_{\mathscr H_\beta} p(H_\beta = h)\rho(h) \Big(\mathds{1}_{\{g(h) \leq \nu\}} \Big) \mathrm{d}h. \label{eqn:6} 
\end{align}
Finally, it can be shown that our proposed estimator $\hat F_n$ is an unbiased estimator of $F_\pi$ by taking the expected value of $\hat F_n$,
\begin{align}
    \mathbb{E}_{\mathcal D} \Big[\hat F_n (\nu) \Big] &= \mathbb{E}_{\mathcal D} \left[ \frac{1}{n}  \sum_{i=1}^n \rho_i\Big( \mathds{1}_{\{G_i \leq \nu\}}\Big) \right]
    \\ 
    &= \frac{1}{n}  \sum_{i=1}^n  \mathbb{E}_{\mathcal D} \left[ \rho_i\Big( \mathds{1}_{\{G_i \leq \nu\}}\Big) \right] 
    \\
    &= \frac{1}{n}  \sum_{i=1}^n \int_{\mathscr{H}_{\beta_i}} p(H_{\beta_i} = h)\rho(h) \Big( \mathds{1}_{\{g(h) \leq \nu\}}\Big) \mathrm{d}h
    \\
    &\overset{(a)}{=} \frac{1}{n} \sum_{i=1}^n F_\pi(\nu)
    \\
    &= F_\pi(\nu), \label{eqn:9}
\end{align}
where (a) follows from \eqref{eqn:6}, which holds for any behavior policy $\beta$ that satisfies \thref{ass:support}.

\textbf{Note:}
%
$H_\pi$ or $H_\beta$ were invoked only for the  purposes of the proof. 
%
Notice that the proposed estimator, $\hat F_n(\nu) = \frac{1}{n}\sum_{i=1}^n \rho_i\big( \mathds{1}_{\{G_i \leq \nu\}}\big)$, only depends on the quantities available in the observed trajectory $(H_i)_{i=1}^n$ from $\mathcal D$. 

\paragraph{Part 2 (Uniform Consistency). } 
For this part, we will first show pointwise consistency, i.e., for any $\nu$, $\hat F_n(\nu) \overset{\text{a.s.}}{\longrightarrow} F_\pi(\nu)$, and then we will use this to establish \textit{uniform} consistency, as required.
To do so, let 
\begin{align}
    X_i \coloneqq  \rho_i\Big( \mathds{1}_{\{G_i \leq \nu\}}\Big). \label{eqn:7}
\end{align}
From \thref{ass:support}, we know that trajectories are independent and that $\beta(a|o) \geq \varepsilon$ when $\pi(a|\tilde o) > 0$.
This implies that the denominator in the IS ratio is bounded below when $\pi(a|\tilde o) \neq 0$, and hence the $X_i$'s are bounded above and have a finite variance. 
Further, as established in \eqref{eqn:9}, the expected value of $X_i$ for all $i$ equals $F_\pi(\nu)$.
Therefore, using Kolmogorov's strong law of large numbers \citep[Theorem  2.3.10 with Proposition 2.3.10]{Sen1993}, 
\begin{align}
    \hat F_n(\nu) = \frac{1}{n} \sum_{i=1}^n X_i \overset{\text{a.s.}}{\longrightarrow} \mathbb{E}_{\mathcal D} \left[ \frac{1}{n} \sum_{i=1}^n X_i\right] = F_\pi(\nu). \label{eqn:10}
\end{align}
In the following, to obtain uniform consistency, we follow the proof for the Glivenko-Cantelli theorem \citep{glivenko1933sulla, cantelli1933sulla,gclemmanote1,gclemmanote2} using the pointwise consistency of the off-policy CDF estimator $\hat F_n$ established in \eqref{eqn:10}.
The proof relies upon the construction of $K$ key points such that the difference in $F_\pi$ at successive key points is bounded by a small $\epsilon_1$. 
However, this would not be possible directly as there can be discontinuties/jumps in $F_\pi$ that are greater than $\epsilon_1$.
To tackle such discontinuties, we introduce some extra notation,
%
%
Formally, let, $\forall \nu \in \mathbb R$,
\begin{align}
    F_\pi(\nu^-) \coloneqq \Pr(G_\pi {\color{red}<} \nu) = F_\pi(\nu) - \Pr(G_\pi=\nu), && \hat F_n(\nu^-) \coloneqq \frac{1}{n} \sum_{i=1}^n \rho_i \Big(\mathds{1}_{\{G_i {\color{red} < }\nu\}} \Big). \label{eqn:20}
\end{align}
Then, using arguments analogous to the ones used for \eqref{eqn:10}, it can be observed that 
\begin{align}
    \hat F_n(\nu^-) \overset{\text{a.s}}{\longrightarrow} F_\pi(\nu^-). \label{eqn:12}
\end{align}
Let $\epsilon_1 > 0$, and let $K$ be any value more than $1/\epsilon_1$.
Let $(\kappa_i)_{i=0}^K$ be $K$ key points, 
\begin{align}
    G_{\min} = \kappa_0 < \kappa_1 \leq \kappa_2  ....  \leq \kappa_{K-1} < \kappa_K = G_{\max},
\end{align}
which create $K$ intervals such that for all $i \in (1, ..., K-1)$,
\begin{align}
    F_\pi(\kappa_i^-) \leq \frac{i}{K} \leq F_\pi(\kappa_i).
\end{align}
Then by construction, if $\kappa_{i-1} < \kappa_i$,
\begin{align}
    F_\pi(\kappa_i^-) - F_\pi(\kappa_{i-1}) \leq \frac{i}{K} - \frac{i-1}{K} = \frac{1}{K} < \epsilon_1. \label{eqn:11}
\end{align}
Intuitively, as $F_\pi$ is monotonically non-decreasing, \eqref{eqn:11} restricts the intermediate values for any $F_\pi(\nu)$, to be within an $\epsilon_1$ distance of the CDF values at its nearby key points.
Notice the role of $\kappa_i^-$ here: it would not have been possible to bound difference between $F_\pi(\kappa_i)$ and $F_\pi(\kappa_{i-1})$ by $\epsilon_1$ as there could have been `jumps' of value greater than $\epsilon_1$ in $F_\pi$.
However, $\kappa^-$ and $\kappa$ can be used to consider key points right before and after any jump in $F_\pi$, which ensures that we can always construct sequence of key points  such that $F_\pi(\kappa_{i}^-) -  F_\pi(\kappa_{i-1})$ is instead bounded by $\epsilon_1$.

For the CDF estimates at the key points, let,
\begin{align}
    \Delta_n \coloneqq \max_{i \in (1...K-1)} \Big\{\left|\hat F_n(\kappa_i) - F_\pi(\kappa_i) \right|, \left|\hat F_n(\kappa_i^-) - F_\pi(\kappa_i^-) \right|  \Big\}. \label{eqn:13}
\end{align}
From \eqref{eqn:10} and \eqref{eqn:12}, as $\hat F_n(\nu)$ and $\hat F_n(\nu^-)$ are consistent estimators of $F_\pi(\nu)$ and $F_\pi(\nu^-)$, respectively, and since the maximum is over a finite set in \eqref{eqn:13}, it follows that as $n \rightarrow \infty$,
\begin{align}
    \Delta_n \overset{\text{a.s.}}{\longrightarrow} 0. \label{eqn:18}
\end{align}
For any $\nu$, let $\kappa_{i-1}$ and $\kappa_i$ be such that $\kappa_{i-1} \leq \nu < \kappa_i$.
Then,
\begin{align}
    \hat F_n(\nu) - F_\pi(\nu) &\leq \hat F_n(\kappa_i^-) - F_\pi(\kappa_{i-1}) 
    \\
    &\leq \hat F_n(\kappa_i^-) - F_\pi(\kappa_{i}^-) + \epsilon_1, \label{eqn:14}
\end{align}
where the last step follows using \eqref{eqn:11}. Similarly, 
\begin{align}
    \hat F_n(\nu) - F_\pi(\nu) &\geq \hat F_n(\kappa_{i-1}) - F_\pi(\kappa_{i}^-) 
    \\
    &\geq \hat F_n(\kappa_{i-1}) - F_\pi(\kappa_{i-1}) - \epsilon_1. \label{eqn:15}
\end{align}
Then, using \eqref{eqn:14} and \eqref{eqn:15}, $\forall \nu \in \mathbb R$,
\begin{align}
   \hat F_n(\kappa_{i-1}) - F_\pi(\kappa_{i-1}) - \epsilon_1 \leq  \hat F_n(\nu) - F_\pi(\nu) \leq \hat F_n(\kappa_i^-) - F_\pi(\kappa_{i}^-) + \epsilon_1, \label{eqn:16}
\end{align}
and thus using \eqref{eqn:13} and \eqref{eqn:16},
\begin{align}
    \Big| \hat F_n(\nu) - F_\pi(\nu)  \Big| \leq \Delta_n + \epsilon_1. \label{eqn:17}
\end{align}
Using \eqref{eqn:18}, we obtain the following property of the upper bound in \eqref{eqn:17}: 
\begin{align}
    \quad\quad \Delta_n + \epsilon_1 \overset{\text{a.s}}{\longrightarrow} \epsilon_1. \label{eqn:19}
\end{align}
Finally, since \eqref{eqn:17} holds for $\forall \nu \in \mathbb R$ and \eqref{eqn:19} is valid for any $\epsilon_1 > 0$, making $\epsilon_1 \rightarrow 0$ gives the desired result,
\begin{align}
    \underset{\nu \in \mathbb R}{\sup} \quad \Big|\hat F_n(\nu) -  F_\pi(\nu) \Big| &\overset{\text{a.s.}}{\longrightarrow} 0. \label{eqn:21}
\end{align}
\end{proof}

\paragraph{Variance-reduced estimation:}
\label{apx:sec:WIS}
It is known that importance-sampling-based estimators are subject to high variance, which can often be limiting in practice \citep{guo2017using}.
A popular approach to mitigate variance is to use \textit{weighted} importance sampling (WIS), which trades off variance for bias.
Leveraging this approach, we propose the following variance-reduced estimator, $\bar F_n$, of $F_\pi$,
\begin{align}
    \forall \nu \in \mathbb R, \quad \bar F_n(\nu) &\coloneqq \frac{1}{\sum_{j=1}^n \rho_j} \left(\sum_{i=1}^n \rho_i \Big(\mathds{1}_{\{G_i \leq \nu\}}\Big) \right).
    \label{eqn:WISF}
\end{align}
%
%
In the following theorem, we show that $\bar F_n$ is a biased estimator of $F_\pi$, though it preserves consistency.
%
%
\begin{prop}
Under \thref{ass:support}, $\bar F_n$ may be biased but is a uniformly consistent estimator of $F_\pi$, 
\begin{align}
    \forall \nu\in \mathbb{R}, \quad \mathbb{E}_{\mathcal D}\Big[\bar F_n(\nu)\Big] &\neq F_\pi, 
    &
    \underset{\nu \in \mathbb R}{\sup} \quad \Big|\bar F_n(\nu) -  F_\pi(\nu) \Big| \overset{\text{a.s.}}{\longrightarrow} 0.
\end{align}
\thlabel{thm:WISFbiased}
\end{prop}
\begin{proof} 
Similar to the proof for \thref{thm:Funbiased}, we break this proof in two parts, one to establish bias and the other to establish consistency of $\hat F_n$.

\paragraph{Part 1 (Biased):} We prove this using a counter-example. Let $n=1$ and $\pi \neq \beta_1$, so
\begin{align}
    \forall \nu\in \mathbb{R}, \quad \mathbb{E}_{\mathcal D}\Big[\bar F_n(\nu)\Big] &= \mathbb{E}_{\mathcal D}\left[ \frac{1}{\sum_{j=1}^1 \rho_j} \left(\sum_{i=1}^1 \rho_i \mathds{1}_{\{G_i \leq \nu\}} \right)\right] 
    \\
    &= \mathbb{E}_{\mathcal D}\left[  \mathds{1}_{\{G_1 \leq \nu\}}\right]
    \\
    &\overset{(a)}{=}  \int_{\mathscr H_{\beta_1}} p(H_{\beta_1} = h)\Big( \mathds{1}_{\{g(h) \leq \nu\}}\Big) \mathrm{d}h
    \\
    &= F_{\beta_1}(\nu)
    \\
    &\neq F_\pi(\nu), 
\end{align}
where (a) follows analogously to \eqref{eqn:2}.

\paragraph{Part 2 (Uniform Consistency): }
First, we will establish pointwise consistency, i.e., for any $\nu$, $\bar F_n(\nu) \overset{\text{a.s.}}{\longrightarrow} F_\pi(\nu)$, and then we will use this to establish \textit{uniform} consistency, as required.
\begin{align}
    \forall \nu \in \mathbb R, \quad \bar F_n(\nu) &=  \frac{1}{\sum_{j=1}^1 \rho_j} \left(\sum_{i=1}^1 \rho_i \mathds{1}_{\{G_i \leq \nu\}} \right)
    \\
    &=  \left(\frac{1}{n}\sum_{j=1}^n \rho_j\right)^{-1} \left(\frac{1}{n}\sum_{i=1}^n \rho_i \mathds{1}_{\{G_i \leq \nu\}} \right).
\end{align}
Let $X_n \coloneqq  \frac{1}{n}\sum_{j=1}^n \rho_j$ and $Y_n \coloneqq \frac{1}{n}\sum_{i=1}^n \rho_i \mathds{1}_{\{G_i \leq \nu\}}$.
Now, as $\bar F_n(\nu)$ is a continuous function of both $X_n$ and $Y_n$, 
if both $( \lim\limits_{n\rightarrow \infty} \, X_n)^{-1}$ and $(\lim\limits_{n \rightarrow \infty} \, Y_n)$ exist then using the continuous mapping theorem \citep[Theorem  2.3]{van2000asymptotic}, 
\begin{align}
    \forall \nu \in \mathbb R, \quad \lim_{n \rightarrow \infty} \quad \bar F_n(\nu) &=   \left(\lim_{n \rightarrow \infty} \, X_n\right)^{-1} \left(\lim_{n \rightarrow \infty} \, Y_n \right) \label{eqn:23}.
\end{align}
Notice using Kolmogorov's strong law of large numbers \citep[Theorem  2.3.10 with Proposition 2.3.10]{Sen1993} that the term in the first parentheses will almost surely converge to the expected value of importance ratios, which equals one \citep{precup2000eligibility}.
Similarly, we know from \eqref{eqn:10} that the term in the second parentheses will converge to $F_\pi(\nu)$ almost surely. Therefore, both parenthetical terms of \eqref{eqn:23} exist, and thus 
\begin{align}
    \forall \nu \in \mathbb R, \quad \bar F_n(\nu) \overset{\text{a.s.}}{\longrightarrow} (1)^{-1} (F_\pi(\nu)) = F_\pi(\nu) . \label{eqn:22}
\end{align}

Now, similar to the proof for \thref{thm:Funbiased}, combining \eqref{eqn:22}  with arguments from \eqref{eqn:20} to \eqref{eqn:21}, it can be observed that 

\begin{align}
    \underset{\nu \in \mathbb R}{\sup} \quad \Big|\bar F_n(\nu) -  F_\pi(\nu) \Big| &\overset{\text{a.s.}}{\longrightarrow} 0.
\end{align}

\end{proof}

\begin{thm}
Under \thref{ass:support}, for any $\delta \in (0, 1]$, if $\sum_{i=1}^K\delta_i \leq \delta$, then the confidence band defined by $F_-$ and $F_+$ provides guaranteed coverage for $F_\pi$.
That is, 
\begin{align}
    \Pr \Big(\forall \nu, \,\, F_{-}(\nu) \leq F_\pi(\nu) \leq F_+(\nu) \Big) \geq 1 - \delta.
\end{align}
\end{thm}

\begin{proof}
 Let $A_i$ be the event that for the key point $\kappa_i$, $\texttt{CI}_-(\kappa_i, \delta_i) \leq F_\pi(\kappa_i) \leq \texttt{CI}_+(\kappa_i, \delta_i)$, for all $i \in (1,..., K)$. Let superscript $c$ denote a complementary event;  then by the union bound, the total probability of the bounds holding at each key point simultaneously is 
 \begin{align}
     \Pr\Big(\cap_{i=1}^K A_i\Big) &= 1 - \Pr\Big((\cap_{i=1}^K A_i)^c \Big)
     = 1 - \Pr\Big(\cup_{i=1}^K A_i^c \Big)
     \geq 1 - \sum_{i=1}^K\Pr\Big( A_i^c \Big) 
     \overset{(a)}{\geq} 1 - \delta,\;\;\;\; \label{eqn:24}
 \end{align}
 where $(a)$ holds because the  conditions of the theorem assert that the sum of probabilities of the bounds failing at each key point is at most $\delta$.
 Therefore, using \eqref{eqn:24},
 \begin{align}
     \Pr\left(\forall i \in (1,...,K), \,\, \texttt{CI}_-(\kappa_i, \delta_i) \leq F_\pi(\kappa_i) \leq \texttt{CI}_+(\kappa_i, \delta_i)  \right) \geq 1 - \delta. \label{eqn:25}
 \end{align}
 Since by construction, at the key points $(\kappa_i)_{i=1}^K, F_-(\kappa_i) = \texttt{CI}_-(\kappa_i, \delta_i)$ and $F_+(\kappa_i) = \texttt{CI}_+(\kappa_i, \delta_i)$, it follows from \eqref{eqn:25} that
 \begin{align}
     \Pr\left(\forall i \in (1,...,K), \,\, F_-(\kappa_i) \leq F_\pi(\kappa_i) \leq F_+(\kappa_i)  \right) \geq 1 - \delta. \label{eqn:26}
 \end{align}
 Using the monotonically non-decreasing property of a CDF, at any point $\nu \in \mathbb R$ such that $\kappa_i \leq \nu \leq \kappa_{i+1}$, we know that $ F_\pi(\kappa_{i}) \leq F_\pi(\nu) \leq F_\pi(\kappa_{i+1})$. Therefore, when the bounds at the key points hold, $F_\pi$ at the key points can also be upper and lower bounded:  $ F_-(\kappa_{i}) \leq F_\pi(\nu) \leq F_+(\kappa_{i+1})$. Therefore, by \eqref{eqn:26} and the construct in \eqref{eqn:lub}, it immediately follows that 
 
\begin{align}
    \Pr \Big(\forall \nu, \,\, F_{-}(\nu) \leq F_\pi(\nu) \leq F_+(\nu) \Big) \geq 1 - \delta.
\end{align}
\end{proof}

\begin{thm}
Under \thref{ass:support}, for any $1-\delta$ confidence band $\mathcal F$, the confidence interval defined by $\psi_-$ and $\psi_+$ provides guaranteed coverage for $\psi(F_\pi)$.
That is,
\begin{align}
\Pr \Big(\psi_- \leq \psi(F_\pi) \leq \psi_+ \Big) \geq 1 - \delta. 
\end{align}
\end{thm}

\begin{proof} Recall that the confidence band $\mathcal F$ is a random variable dependent on the data $\mathcal D$.
Let $\mathbb{E}_{\mathcal F}[\cdot]$ represent expectation with respect to $\mathcal F$, then repeatedly using the law of total probability, 
    \begin{align}
        \Pr\Big(\psi_- \leq \psi(F_\pi) \leq \psi_+ \Big) 
        &= \mathbb{E}_{\mathcal F}\left[\Pr\Big(\psi_- \leq \psi(F_\pi) \leq \psi_+ \Big | \mathcal F \Big) \right] \\ 
        &= \mathbb{E}_{\mathcal F}\Big[\Pr\Big(\psi_- \leq \psi(F_\pi) \leq \psi_+ \Big| F_\pi \in \mathcal F, \mathcal F\Big) \Pr \Big(F_\pi \in \mathcal F \Big|  \mathcal F\Big) \\ &\quad\quad\quad + \Pr\Big(\psi_- \leq \psi(F_\pi) \leq \psi_+ \Big| F_\pi \not \in \mathcal F, \mathcal F\Big) \Pr\Big(F_\pi \not \in \mathcal F \Big | \mathcal F\Big) \Big] 
        \\
        &\geq \mathbb{E}_{\mathcal F}\left[\Pr\Big(\psi_- \leq \psi(F_\pi) \leq \psi_+ \Big| F_\pi \in \mathcal F, \mathcal F \Big) \Pr\Big(F_\pi \in \mathcal F\Big |  \mathcal F\Big) \right] 
        \\
        &\overset{(a)}{=} \mathbb{E}_{\mathcal F}\left[\Pr\Big(F_\pi \in \mathcal F \Big | \mathcal F\Big) \right]
        \\
        &= \Pr\Big(F_\pi \in \mathcal F \Big)
        \\
        &\overset{(b)}{\geq} 1 - \delta,
    \end{align}
    where $(a)$ follows from that fact that $F_\pi \in \mathcal F$ implies $\psi_- \leq \psi(F_\pi) \leq \psi_+ $. Step $(b)$ follows from \thref{thm:Fguarantee1}.
\end{proof}

\begin{proof}[Proof (Alternate)] This proof is shorter but requires a theoretical construct of a \textit{set of sets of functions}. 
That is, let $\mathbb F$ be any set of cumulative distribution functions and $\mathscr F$ be a set of such sets, such that
\begin{align}
    \mathscr F \coloneqq \Big\{ \mathbb F \,\, \Big| F_\pi \in \mathbb F \Big\}.
\end{align}
In other words, $\mathbb F$ is the set of CDFs which contains the true CDF $F_\pi$, and $\mathscr F$ is the set of \textit{all} such sets $\mathbb F$. From \thref{thm:Fguarantee1}, we know that the confidence band $\mathcal F$ contains $F_\pi$ with probability at least $1-\delta$. Therefore, it also holds that 
\begin{align}
    \Pr(\mathcal F \in \mathscr F) \geq 1-\delta.
\end{align}
However, the event $(\mathcal F \in \mathscr F)$ implies that $\psi_- \leq \psi(F_\pi) \leq \psi_+$ as $F_\pi$ is contained in this specific  $\mathcal F$ used to construct $\psi_-$ and $\psi_+$.
Therefore, it also holds that 
\begin{align}
    \Pr(\psi_- \leq \psi(F_\pi) \leq \psi_+) \geq 1 - \delta.
\end{align}
    
\end{proof}

\begin{thm}
Under \thref{ass:support,ass:shift}, for any $\delta \in (0, 1]$, the confidence band defined by $F^{(2)}_-$ and $F^{(2)}_+$ provides guaranteed coverage for $F^{(2)}_\pi$.
That is, 
\begin{align}
    \Pr \Big(\forall \nu, \,\, F^{(2)}_{-}(\nu) \leq F^{(2)}_\pi(\nu) \leq F^{(2)}_+(\nu) \Big) \geq 1 - \delta.
\end{align}
\end{thm}

\begin{proof}
From \thref{ass:shift},  $\underset{\nu \in \mathbb R}{\sup} \left | F^{(1)}_\pi(\nu) - F^{(2)}_\pi(\nu) \right| \leq \epsilon$. Or equivalently, 
\begin{align}
    \forall \nu \in \mathbb{R}, \quad F_\pi^{(1)}(\nu) - \epsilon \leq F_\pi^{(2)}(\nu) \leq F_\pi^{(1)}(\nu) + \epsilon. \label{eqn:27}
\end{align}
Using \thref{thm:Fguarantee1} for the bound obtained on $F_\pi^{(1)}$ for the first domain,
\begin{align}
    \Pr \Big(\forall \nu, \,\, F_{-}^{(1)}(\nu) \leq F_\pi^{(1)}(\nu) \leq F_+^{(1)}(\nu) \Big) \geq 1 - \delta. \label{eqn:28}
\end{align}
Therefore, combining \eqref{eqn:27} and \eqref{eqn:28}, 
\begin{align}
    \Pr \Big(\forall \nu, \,\, F_{-}^{(1)}(\nu) - \epsilon \leq F_\pi^{(2)}(\nu) \leq F_+^{(1)}(\nu) + \epsilon \Big) \geq 1 - \delta. \label{eqn:29}
\end{align}
Then by the construct in \eqref{eqn:Fshiftbound}, it follows from \eqref{eqn:29} that
\begin{align}
    \Pr \Big(\forall \nu, \,\, F_{-}^{(2)}(\nu) \leq F_\pi^{(2)}(\nu) \leq F_+^{(2)}(\nu) \Big) \geq 1 - \delta.
\end{align}
\end{proof}

\section{Extended Discussion for UnO}

\subsection{Nuances for CDF Inverse and CVaR}
\label{apx:UnOinverse}
For brevity, some nuances for $\hat F_n^{-1}(\alpha)$ and $\text{CVaR}_\pi^\alpha(\hat F_n)$ were excluded from the main paper.
We discuss them in this section.

As discussed earlier in \thref{rem:geq1}, it is possible that $\hat F_n(\nu) > 1$ for some $\nu \in \mathbb R$ due to the use of importance weighting.
Similarly, it is also possible that $\hat F_n(\nu) < 1$ for all $\nu \in \mathbb R$.
Specifically, if $\hat F_n(\nu) < \alpha$ for all $\nu$, then it raises the question: how can one obtain an estimate of $ F_\pi^{-1}(\alpha)$?
To resolve this issue, we use the following estimator of $ F_\pi^{-1}(\alpha)$ for UnO:
\begin{align}
    \hat F^{-1}_n(\alpha) \coloneqq \begin{cases}
    \min \Big\{g \in (G_{(i)})_{i=1}^n \Big| \hat F_n(g) \geq \alpha \Big\}, & \text{if} \quad  \exists \, g  \,\,\text{s.t.}\,\, \hat F_n(g) \geq \alpha,
    \\
    \max (G_{(i)})_{i=1}^n  & \text{otherwise.}
    \end{cases}
    \label{apx:eqn:inverseF}
\end{align}
However, it is known from \thref{thm:Funbiased} that $\hat F_n$ is a uniformly consistent estimator of $F_\pi$.
Therefore, the edge case that $\hat F_n(\nu) < \alpha$ for all $\nu$ cannot occur in the limit as $n \rightarrow \infty$.
Resolving this is required mostly when the sample size is small.

Regarding CVaR, it is known \citep{acerbi2002coherence} that when the distribution of a random variable (which is $G_\pi$ for UnO) is continuous, then CVaR can be expressed as,
\begin{align}
    \text{CVaR}^\alpha_\pi(F_\pi) &= \mathbb{E}\left[G_\pi \middle| G_\pi \leq F_\pi^{-1}(\alpha) \right], \label{eqn:apx:cvar1}
    \end{align}
and thus an off-policy sample estimator for \eqref{eqn:apx:cvar1} can be constructed as,    
    \begin{align}
    {\text{CVaR}}^\alpha_\pi(\hat F_n) &\coloneqq \frac{1}{\alpha}\sum_{i=1}^n \text{d}\hat F_n(G_{(i)}) G_{(i)} \mathds{1}_{\left\{G_{(i)} \leq {Q}^\alpha_\pi(\hat F_n)\right\}}.
\end{align}
However, for distributions that are not continuous, a more generic definition for CVaR is \citep{brown2007large},
\begin{align}
    \text{CVaR}^\alpha_\pi(F_\pi) &= \inf_{g} \left\{g - \frac{1}{\alpha}\mathbb{E}\Big[\max\big(0,g - G_\pi \big) \Big] \right\}. \label{eqn:apx:cvar2}
\end{align}
We extend the sample estimator by \citet{brown2007large} for \eqref{eqn:apx:cvar2} and use the following off-policy estimator for UnO:
\begin{align}
    \text{CVaR}^\alpha_\pi(\hat F_n) &\coloneqq \hat F_n^{-1}(\alpha) - \frac{1}{\alpha}\sum_{i=1}^{n} \text{d}\hat F_n(G_{(i)}) \left(\max\big(0, \hat F_n^{-1}(\alpha) - G_{(i)}\big) \right)
\end{align}

\subsection{Optimizing Confidence Bands for Tighter Bounds: }
Constructing $\mathcal F$ requires selecting $K$ key points for which $\texttt{CI}$s are computed.
If too many key points are selected, then each $\delta_i$ has to be a very small positive value so that $\sum_{i=1}^K \delta_i \leq \delta$, as required by \thref{thm:Fguarantee1}.
This will make the confidence intervals wide at each key point. 
In contrast, if too few key points are selected, then the confidence intervals at the $\kappa_i$'s will be relatively tighter, but this will not tighten the intervals \textit{between} the $\kappa_i$'s due to the way $F_-$ and $F_+$ are constructed in \eqref{eqn:Fband2}.
Further, the overall tightness of $\mathcal F$ is also affected by the location of each $\kappa_i$ and its respective failure rate $\delta_i$.  
Therefore, to get a tight $\mathcal F$, we propose searching for a $\theta \coloneqq \left(K,(\kappa_i)_{i=1}^K,   (\delta_i)_{i=1}^K\right)$ that minimizes the area enclosed in $\mathcal F$.
That is, let $\Delta_{i+1} \coloneqq \kappa_{i+1} - \kappa_i$, then the area enclosed in $\mathcal F$ is
\begin{align}
    \mathscr A(\theta) \coloneqq \sum_{i=0}^{K} \left(\texttt{CI}_+(\kappa_{i+1}, \delta_{i+1})- \texttt{CI}_-(\kappa_{i}, \delta_{i}) \right) \Delta_{i+1}. 
\end{align}
To avoid multiple comparisons \citep{benjamini1995controlling}, we first partition $\mathcal D$ into $\mathcal D_\text{train}$ and $\mathcal D_\text{eval}$. Subsequently, $\mathcal D_\text{train}$ is used to search for $\theta^*$ as follows, and then $\theta^*$ is used with $\mathcal D_\text{eval}$ to obtain $\mathcal F$. 
\begin{align}
    \theta^* \coloneqq \,\,&  \underset{\theta}{ \argmin} \,\, \mathscr A(\theta) \label{eqn:Foptim}
    \\
    \text{s.t.} \quad\quad&   G_{\min} < \kappa_i <  G_{\max}, \quad  \sum_{i=1}^K \delta_i \leq \delta, \quad  \delta_i \geq 0, \quad & \forall i \in (1,...,K). 
\end{align}
\begin{rem}
    A global optimum of \eqref{eqn:Foptim} is not required---any feasible $\theta$ can be used with $\mathcal D_\text{eval}$ to obtain a confidence band $\mathcal F$. Optimization only helps by making the band tighter.
\end{rem}

For our experimental results, when searching $\theta^*$ for \eqref{eqn:Foptim}, we keep the number of key points, $K$, fixed to $\log(n)$, where $n$ is the number of observed trajectory samples in $\mathcal D$. To search for the locations $(\kappa_i)_{i=1}^K$ and the failure rates $(\delta_i)_{i=1}^K$ at each key point, we use the BlackBoxOptim library\footnote{https://github.com/robertfeldt/BlackBoxOptim.jl} available in Julia \citep{bezanson2017julia}. 
To perform this optimization, we construct $\mathcal D_\text{train}$ using $5\%$ of data from $\mathcal D$, and construct $\mathcal D_\text{eval}$ using the rest of the data.
Following the idea by \citet{thomas2015higha}, when searching for $\theta^*$ using $\mathcal D_\text{train}$, bounds for the key points $(\kappa_i)_{i=1}^K$ are obtained as if the number of samples are equal to the number of samples available in $\mathcal D_\text{eval}$
(see Equation $7$ in the work by \citet{thomas2015higha} for more discussion on this).
Instead of using a single split, one could potentially also leverage results by \citet{romano2019multiple} to use multiple splits; we leave this for future work.

\subsection{Bound Specialization}
\label{apx:UnOspecial}

In \eqref{eqn:Foptim}, $\theta$ was searched to minimize the area $\mathscr A(\theta)$ enclosed within $\mathcal F(\theta)$, where $\mathcal F(\theta)$ represents the CDF band obtained using the parameter $\theta$.
This was done without any consideration of the downstream parameter $\psi$ for which the bounds would be constructed using $\mathcal F(\theta)$.
Therefore, the band $\mathcal F(\theta)$ is tight overall, but need not be the best possible if only a specific parameter $\psi$'s bounds are required using $\mathcal F(\theta)$.

For example, consider obtaining bounds for $\text{CVaR}_\pi^\alpha$.
As can be seen from the geometric insight in Figure \ref{fig:geometry}, bounds for CVaR are mostly dependent on the tightness of $\mathcal F(\theta)$ near the lower tail.
Therefore, if one can obtain $\mathcal F(\theta)$ that is tighter near the lower tail, albeit looser near the upper tail, that would provide a better bound for CVaR as opposed to a band $\mathcal F(\theta)$ that has uniform tightness throughout.

To get a tight $\mathcal F(\theta)$ in such cases where there is a single downstream parameter of interest,  we propose searching for a $\theta \coloneqq \left(K,(\kappa_i)_{i=1}^K,   (\delta_i)_{i=1}^K\right)$ that directly optimizes for the final parameter of interest instead of the area enclosed in $\mathcal F(\theta)$.
For example, if only the lower bound for $\psi(F_\pi)$ is required, then let 
\begin{align}
    \mathscr \psi_-(\theta) &\coloneqq \underset{F \in \mathcal F(\theta)}{\inf} \,\,\, \psi (F).
\end{align}
Next, the optimization using $\mathcal D_\text{train}$ can then be modeled as the following,
\begin{align}
    \theta^* \coloneqq \,\,&  \underset{\theta}{ \argmax} \,\, \mathscr \psi_-(\theta)
    \\
    \text{s.t.} \quad\quad&   G_{\min} < \kappa_i <  G_{\max}, \quad & \forall i \in (1,...,K), 
    \\
    &  \sum_{i=1}^K \delta_i \leq \delta, \quad  \delta_i \geq 0, & \forall i \in (1,...,K),
\end{align}
This would result in $\theta^*$ that when used with $\mathcal D_\text{eval}$ can be expected to provide the CDF band which will yield the highest lower bound for $\psi(F_\pi)$.

\subsection{Approximate Bounds for Any Parameter using Bootstrap}
\label{apx:sec:boot}

In Algorithm \ref{alg:pboot}, we provide the pseudo code for obtaining bootstrap-based bounds for any parameter $\psi(F_\pi)$. In Line 1, $B$ datasets $(\mathcal D_i^*)_{i=1}^B$ are generated from $\mathcal D$ using resampling, and for each of these resampled data sets, $B$ (weighted IS-based) CDF estimates  $(\bar F^{*}_{n,i})_{i=1}^B$ are obtained. 
In Line 3,  sample estimates $(\psi(\bar F^{*}_{n,i}))_{i=1}^B$  for the desired parameter $\psi(F_\pi)$ are constructed using the $B$ estimated CDFs.
In Line 4, these sample estimates for $\psi(F_\pi)$ can be subsequently passed to the bias-corrected and accelerated (BCa   \citep{efron1994introduction}) bootstrap procedure to obtain approximate lower and upper bounds $(\psi_-, \psi_+)$.

	\IncMargin{1em}

	\begin{algorithm2e}[h]
		\textbf{Input:} {Dataset $\mathcal D$}, Confidence level $1 - \delta$ 
		\\
		Bootstrap $B$ datasets $(\mathcal D_i^*)_{i=1}^B$ and create $(\bar F^{*}_{n,i})_{i=1}^B$ 
		\\
		Bootstrap  estimates $(\psi(\bar F^{*}_{n,i}))_{i=1}^B$ using $(\bar F_{n,i}^*)_{i=1}^B$
		\\
	 
	Compute $(\psi_-, \psi_+)$ using BCa($(\psi(\bar F^{*}_{n,i}))_{i=1}^B$, $\delta$)  
		\\
		\textbf{Return} $(\psi_-, \psi_+)$
		\caption{Bootstrap Bounds for $\psi(F_\pi)$}
	\label{alg:pboot}
	\end{algorithm2e}
	\DecMargin{1em}

\subsection{Extended Discussion of High-Confidence Bounds for Any Parameter}
\label{apx:UnOoptim}

Section \ref{sec:Fbounds} of the main paper discussed how high-confidence bounds $\psi_-$ and $\psi_+$ can be obtained for any parameter $\psi(F_\pi)$ using the confidence band $\mathcal F$.
Specifically, in Figure \ref{fig:geometry}, geometric insights for obtaining the analytical form of the bounds for the mean, quantile, and CVaR were discussed.
Extending that discussion, Figure \ref{apx:fig:geometry} provides geometric insights for bounding other parameters, namely variance, inter-quantile ranges, and entropy, in the off-policy setting.

An advantage of having the CDF band $\mathcal F$ is that it can permit bounding other novel parameters that might be of interest.
While analytical bounds using geometric insights, as discussed for a number of popular parameters, should also be the first attempt for the desired novel parameter, it may be the case that such geometric insight cannot be obtained.
In such cases, a CDF $F$ can be directly parameterized using a spline curve, or a piecewise non-decreasing function that is constrained to be within $\mathcal F$.
%
Depending on how rich this parameterization is, it may be feasible to use a black-box optimization routine and obtain a globally optimal $F$ that minimizes (maximizes) the desired parameter $\psi(F)$.
If not feasible, an approximate bound can be achieved by using the best found local optima. 

\begin{figure*}[!ht]
    \centering
    \includegraphics[width=0.32\textwidth]{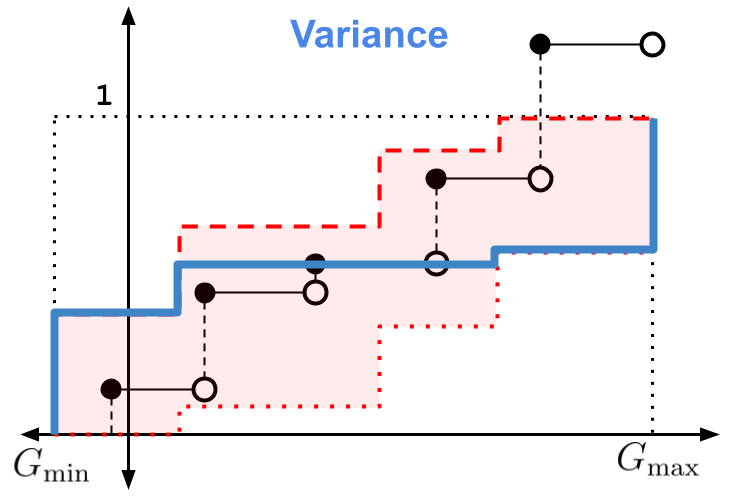}
    \includegraphics[width=0.32\textwidth]{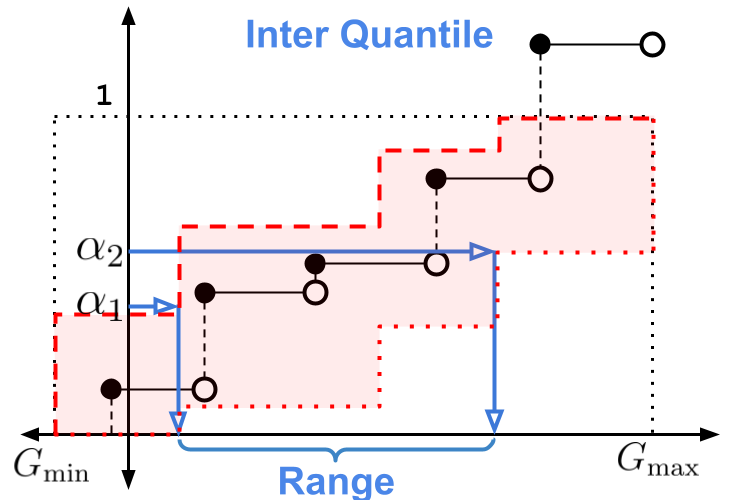}    \includegraphics[width=0.32\textwidth]{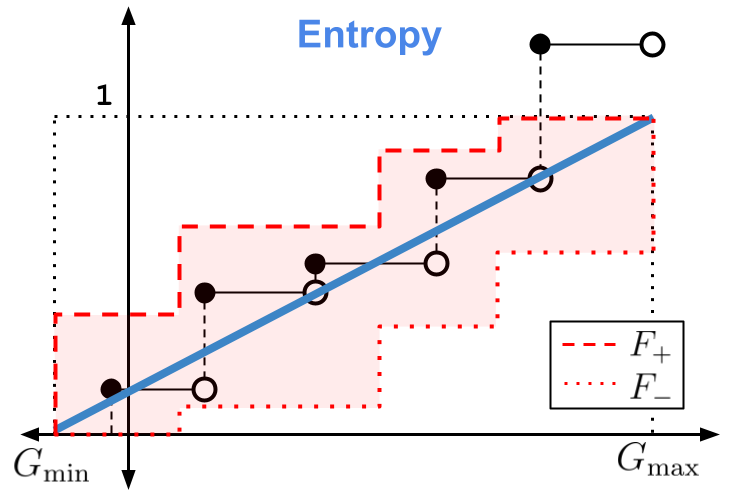}
    \caption{Similar to Figure \ref{fig:geometry}, given a confidence band $\mathcal F$, lower and upper bounds for several other parameters can also be obtained using simple geometric insights.
    \textbf{(Left)} An upper bound for the variance can be obtained by observing that variance is maximized when the probability of events on either extreme are maximized. Therefore, the CDF $F \in \mathcal F$ for such a distribution will initially follow (from left to right) $F_+$ and then make a horizontal jump (at a specific jump point) to $F_-$, which it then follows until $1$. The variance of the distribution with this CDF, $F$, will give the desired upper bound. Analogously, the CDF that initially follows $F_-$ and then jumps vertically (at a specific jump point) to $F_+$, assigns highest probability to events near the mean and thus results in the lowest variance \citep{romano2002explicit}. 
    \textbf{(Middle)} An upper bound for the inter-quantile range can be obtained by maximizing the value of upper $\alpha_2$-quantile and subtracting the minimum value for the lower $\alpha_1$-quantile. This can be obtained by $F^{-1}_-(\alpha_2) - F^{-1}_+(\alpha_1)$.
    Analogously, a lower bound can be obtained using $\max(0, F^{-1}_+(\alpha_2) - F^{-1}_-(\alpha_1))$.
    \textbf{(Right)} An upper bound on the entropy can be obtained by what \citet{learned2008probabilistic} call a ``string-tightening'' algorithm. That is, if the ends of a tight string are held at the bottom-left and the upper-right corner of $\mathcal F$, and the entire string is constrained to be within $\mathcal F$, then the path of the string corresponds to the $F \in \mathcal F$ that has highest entropy. In our figure, such an $F$ corresponds to the CDF of the uniform distribution, which is known to have maximum entropy. Unless some stronger assumptions are made, the lower bound on differential entropy is typically $- \infty$ if there is any possibility of a point mass.}
    \label{apx:fig:geometry}
\end{figure*}

\subsection{Tackling Smooth Non-stationarity using Wild Bootstrap}
\label{apx:sec:NSboot}

From \thref{thm:Funbiased}, it is known that the proposed estimator $\hat F_n(\kappa)$ provides unbiased estimates for $F_\pi(\kappa)$, even with a single observed trajectory.
In the non-stationary setting, let the true underlying CDF of returns for $\pi$ in the episode $i$ be $F^{(i)}_\pi(\kappa)$, and the estimate of $F^{(i)}_\pi(\kappa)$ using the trajectory observed during the episode $i$ be
\begin{align}
    \hat F_n^{(i)}(\kappa) &\coloneqq \rho_{i} \mathds{1}_{\{G_i \leq \kappa\}} &\forall i \in \{1,2,...,L\}.
\end{align}
Next, the trend of the sequence $\big(\hat F_n^{(i)}(\kappa) \big)_{i=1}^L$ can be analyzed to forecast $\hat F_n^{(L+\ell)}(\kappa)$ for the future episode $L + \ell$ when the policy $\pi$ will be executed.
Particularly, under \thref{ass:ns},  $\exists w_\kappa$, such that, $\forall i \in (1,...,L+\ell), \,\,\,\, F^{(i)}_\pi(\kappa) = \phi(i)^\top w_\kappa.$ 
Therefore, using the unbiased estimates $\big(\hat F_n^{(i)}(\kappa) \big)_{i=1}^L$ of $\big(F_\pi^{(i)}(\kappa) \big)_{i=1}^L$, we propose searching for $w_\kappa$ using least-squares regression. Let $X\coloneqq [1, 2, ...., L] $ be the episode numbers in the past, then the predicates $\Phi_\kappa$, the targets $Y_\kappa$, and the corresponding least-squares solution $w_\kappa$ can be obtained as, 
\begin{align}
    \Phi_\kappa &\coloneqq [\phi(X_1), \phi(X_2), ..., \phi(X_L)] & \in \mathbb{R}^{L \times d},\\
    Y_\kappa &\coloneqq [\hat F_n^{(1)}(\kappa), \hat F_n^{(2)}(\kappa), ..., \hat F_n^{(L)}(\kappa)] & \in \mathbb{R}^{L\times 1},\\
    w_\kappa &\coloneqq \left( \Phi_\kappa^\top \Phi_\kappa \right)^{-1}\Phi_\kappa^\top Y_\kappa & \in \mathbb{R}^{d \times 1}.
\end{align}
Using $w_\kappa$, an unbiased estimate of $F_\pi^{(L+\ell)}(\kappa)$ can be obtained as,
\begin{align}
    \hat F_n^{(L+\ell)}(\kappa)  &\coloneqq \phi(L+\ell)^\top w_\kappa.  \label{eqn:foreast}
\end{align}
The point forecast $\hat F_n^{(L+\ell)}(\kappa)$ from \eqref{eqn:foreast} can then be combined with Algorithms 1 and 2 presented by \citet{chandak2020towards} to obtain wild-bootstrap-based confidence intervals for $F_\pi^{(L+\ell)}(\kappa)$. Once the confidence intervals are obtained at different key points, \eqref{eqn:Fband2} can be used to construct an entire confidence band for $F_\pi^{(L+\ell)}$.

\section{Empirical Details}
\label{apx:empirical}

\subsection{Domain Details}
\label{apx:empiricaldomain}

In this section, we discuss domain details and how $\pi$ and $\beta$ were selected for  these domains.
The code for the domains, baselines \citep{thomas2015higha,chandak2021hcove}, and the proposed UnO estimator can be found at \href{https://github.com/yashchandak/UnO}{https://github.com/yashchandak/UnO}.

\paragraph{Recommender System: }
Systems for online recommendation of tutorials, movies, advertisements, etc., are ubiquitous  \citep{theocharous2015ad,theocharous2020reinforcement}.
In these settings, it may be beneficial to fully characterize a customer's experience once the new system/policy is deployed.
To abstract such settings, we created a simulated domain where the user's interest for a finite set of items is represented using the corresponding item's reward.

Using an actor-critic algorithm \citep{SuttonBarto2}, we find a near-optimal policy $\pi$, which we use as the evaluation policy.
Let $\pi^\texttt{rand}$ be a random policy with uniform distribution over the actions (items).
Then for an $\alpha = 0.5$, we define the behavior policy $\beta(a|s) \coloneqq \alpha \pi(a|s) + (1-\alpha) \pi^\texttt{rand}(a|s)$ for all states and actions.
\paragraph{Gridworld: }
We also consider a standard continuous-state Gridworld with partial observability (which also makes the domain non-Markovian in the observations), stochastic transitions, and eight discrete actions corresponding to up, down, left, right, and the four diagonal movements.
The off-policy data was collected using two different behavior policies, $\beta_1$ and $\beta_2$, and the evaluation policies for this domain were obtained similarly as for the recommender system domain discussed above.
Particularly, using $\alpha = 0.5$, we define  $\beta_1(a|o) \coloneqq \alpha \pi(a|0) + (1-\alpha) \pi^\texttt{rand}(a|o)$ for all states and actions. Similarly, $\beta_2$ was defined using $\alpha = 0.75$.

\paragraph{Diabetes Treatment: } 
This domain is modeled using an open source implementation \citep{simglucose} of the U.S. Food and Drug Administration (FDA) approved Type-1 Diabetes Mellitus Simulator (T1DMS) \citep{man2014uva} for the treatment of type-1 diabetes.
An episode corresponds to a day, and each step of an episode corresponds to a minute in an \textit{in silico} patient's body and is governed by a continuous time nonlinear ordinary differential equation (ODE) \citep{man2014uva}.
In such potentially critical medical applications, it is important to go beyond just the expected performance and to characterize the risk associated with it, \textit{before deployment}. 

To control the insulin injection, which is required for regulating the blood glucose level, we use a policy that controls the parameters of a \textit{basal-bolus controller}. This controller is based on the amount of insulin that a person with diabetes is instructed to inject prior to eating a meal \citep{bastani2014model}:
\begin{align}
    \text{injection} = \frac{\text{current blood glucose} - \text{target blood glucose}}{CF} + \frac{\text{meal size}}{CR},
\end{align}
where ``current blood glucose'' is the estimate of the person's current blood glucose level, ``target blood glucose'' is the desired blood glucose, ``meal size'' is the estimate of the size of the meal the patient is about to eat, and $CR \in [CR_\texttt{min}, CR_\texttt{max}]$ and $CF \in [CF_\texttt{min}, CF_\texttt{max}]$ are two parameters of the controller that must be tuned based on the body parameters to make the treatment effective.
We designed an RL policy that acts on the discretized space of the parameters, $CR$ and $CF$, for the above basal-bolus controller.
Behavior and evaluation policies were selected similarly as discussed for the recommender system domain.

\subsection{Extended Discussion on Results for Stationary Settings}
\label{apx:empiricalS}

The main results for the stationary setting are provided in Figure \ref{fig:results} of the main body.
In this section, we provide some additional discussion on the observed trends for the bounds.

Notice in Figure \ref{fig:results} that UnO-CI bounds for the variance can require up to an order of magnitude less data compared to the existing bound for the variance \citep{chandak2021hcove}.
This can be attributed to the fact that \citet{chandak2021hcove} construct the bounds using $\mathbb{E}[\rho G^2] - \mathbb{E}[\rho G]^2$, where it can be observed that the second term depends quadratically on $\rho$. This makes the variance of that term effectively ``doubly exponential'' in the horizon length. This does not happen in the CDF-based approach as the bounds at any key point $\kappa$ depend on  $\mathbb{E}[\rho \mathds{1}_{G<\kappa}])$, which does not have any higher powers of $\rho$.

Another thing worth noting in Figure \ref{fig:results} is that not only the bounds for different parameters, but even the upper and lower bounds for the same parameter converge at different rates (especially for smaller values of $n$).
Therefore, there are two particular trends to observe:
(a) how close the bounds are to the true value at the beginning, and (b) how quickly they improve.  
Both of these depend on the direction for which clipping plays a major role and also how the bounds depend on the tails.
For example, for the mean, as the distributions are right skewed (because evaluating policy $\pi$ is a near-optimal policy), the bounds on the CDF are clipped more from the lower end (so that $F(\nu) >= 0$ always). 
Therefore, since the upper bound on the mean depends on the lower CDF bound (see Figure \ref{fig:geometry}), it starts close to the estimate itself but the progress actually seems slow because shrinking CDFs bounds at any specific $F(\nu)$ from the lower end does not impact the bound until the point where clipping is not required anymore.

For variance, the upper bound depends on both the upper bound on the lower tail and the lower bound on the upper tail (see Figure \ref{apx:fig:geometry}), and these two benefit from clipping the least and also converge the slowest. 
In contrast, the lower bound for variance depends on the upper bound on the upper tail and the lower bound on the lower tail, which are clipped immediately to be below 1 and above 0, respectively. 
Appendix \ref{apx:ass} (knowledge of $G_{\min}$, $G_{\max}$) and Fig \ref{apx:fig:geometry} provide more intuition on this.